\documentclass[conference]{IEEEtran}

\usepackage[numbers]{natbib}
\usepackage{multicol}
\usepackage[bookmarks=true]{hyperref}
\usepackage[symbol]{footmisc}

\usepackage{times}
\usepackage{graphicx}
\usepackage[table,xcdraw]{xcolor}
\usepackage{algorithm}
\usepackage{algorithmicx}
\usepackage[noend]{algpseudocode}
\usepackage{amsmath}
\usepackage{amsthm}
\usepackage{booktabs}
\usepackage{graphicx}
\usepackage{amssymb} 
\usepackage{multicol}
\usepackage{hyperref}
\usepackage{bbm}
\usepackage{wrapfig}
\usepackage{arydshln}
\usepackage[font=small,labelfont=bf]{caption}
\usepackage{balance}
\usepackage{etoolbox}
\usepackage{sidecap}

\newtheorem{theorem}{Theorem}

\usepackage{xcolor}
\usepackage{xspace}
\usepackage{bbm}
\usepackage{makecell}

\definecolor{orange(sae/ece)}{rgb}{1.0, 0.49, 0.0}
\definecolor{teal(sae/ece)}{rgb}{0, 0.47, 0.52}
\definecolor{purple}{rgb}{0.74, 0.65, 1.0}
\definecolor{dark_purple}{rgb}{0.72, 0.33, 0.82}
\definecolor{light_gray}{rgb}{0.5, 0.5, 0.5}
\definecolor{light_green}{rgb}{0.61,	0.79, 0.66}
\definecolor{light_blue}{rgb}{0, 0.36, 1}
\definecolor{sky_blue}{rgb}{0.51, 0.906, 0.91}
\definecolor{tan}{rgb}{1, 0.78, 0}

\newcommand{\cba}{\textcolor{tan}{\textbf{CBA}}\xspace}

\newcommand{\cadl}{\textcolor{light_green}{\textbf{C-ADL}}\xspace}
\newcommand{\coil}{\textcolor{light_blue}{\textbf{COIL}}\xspace}
\newcommand{\coilnoadapt}{\textcolor{sky_blue}{\textbf{COIL-NoAdapt}}\xspace}
\newcommand{\info}{\textcolor{orange}{\textbf{IG}}\xspace}

\newcommand{\para}[1]{\smallskip \noindent \textbf{{#1}.}}

\def\ind{\mathbbm{1}}

\usepackage{amsmath,amsfonts,bm}

\def\secref#1{section~\ref{#1}}

\def\eqref#1{equation~\ref{#1}}

\def\1{\bm{1}}

\DeclareMathAlphabet{\mathsfit}{\encodingdefault}{\sfdefault}{m}{sl}
\SetMathAlphabet{\mathsfit}{bold}{\encodingdefault}{\sfdefault}{bx}{n}

\newcommand{\seq}[4]{(#1_#2)_{#2=#3}^{#4}}

\DeclareMathOperator*{\argmax}{arg\,max}

\newif\ifshowappendix
\showappendixtrue 

\begin{document}

\title{Optimal Interactive Learning on the Job via \\
Facility Location Planning}


\author{
\authorblockN{
Shivam Vats\authorrefmark{2}\authorrefmark{1} \qquad
Michelle Zhao\authorrefmark{3}\authorrefmark{1} \qquad
Patrick Callaghan\authorrefmark{3} \qquad
Mingxi Jia\authorrefmark{2} \\
Maxim Likhachev\authorrefmark{3} \qquad
Oliver Kroemer\authorrefmark{3} \qquad
George Konidaris\authorrefmark{2}}
\authorblockA{\authorrefmark{2}Brown University \qquad
\authorrefmark{3}Carnegie Mellon University \qquad
\authorrefmark{1}Equal contribution}
}

\maketitle

\begin{abstract}
Collaborative robots must continually adapt to novel tasks and user preferences without overburdening the user.
While prior interactive robot learning methods aim to reduce human effort, they are typically limited to single-task scenarios and are not well-suited for sustained, multi-task collaboration.
We propose COIL (Cost-Optimal Interactive Learning)---a multi-task interaction planner that minimizes human effort across a sequence of tasks by strategically selecting among three query types (skill, preference, and help).
When user preferences are known, we formulate COIL as an uncapacitated facility location (UFL) problem, which enables bounded-suboptimal planning in polynomial time using off-the-shelf approximation algorithms.
We extend our formulation to handle uncertainty in user preferences by incorporating one-step belief space planning, which uses these approximation algorithms as subroutines to maintain polynomial-time performance.
Simulated and physical experiments on manipulation tasks show that our framework significantly reduces the amount of work allocated to the human while maintaining successful task completion.
\end{abstract}

\IEEEpeerreviewmaketitle

\footnotetext{Correspondence to \texttt{shivam\_vats@brown.edu}.}

\section{Introduction}
Collaborative robots hold the promise of reducing human workload, increasing productivity, and improving quality of life by assisting with tedious and labor-intensive tasks.
For human-robot collaboration to be truly effective, robots must safely complete assigned tasks according to individual user preferences.
Although robots can be equipped with prior knowledge about the user and a library of common skills, they will inevitably face novel situations in practical long-term deployments. How can they learn and adapt on the job when faced with tasks that are beyond their capabilities or when the user preference is unclear? We explore these questions in the context of multi-task domains, such as factories and households, that require collaboration on a variety of tasks during deployment. The central challenge in this setting is to enable robots to learn interactively from humans while ensuring that all tasks are completed with minimum human effort.

Prior interactive learning approaches leverage active querying strategies to proactively identify queries that result in maximum uncertainty reduction~\citep{chernovaConfidencebasedPolicyLearning2007}, volume removal~\citep{sadighActivePreferenceBasedLearning2017} or information gain~\citep{biyik2022learningfixed}. Such active learning strategies have been shown to reduce the number of queries.
Furthermore, the use of multiple query types, such as demonstrations and preference queries, has been shown to alleviate the teaching burden by accounting for the cost associated with each query type~\citep{fitzgeraldINQUIREINteractiveQuerying2022}.
However, these approaches focus on single tasks, and therefore do not guarantee minimization of human effort across multiple tasks in extended collaboration.
The latter challenges the robot to reason proactively about the expected utility of its queries so that it does not overburden the user by seeking to learn all the tasks.

\begin{figure}[t]
    \centering
    \includegraphics[width=1\columnwidth]{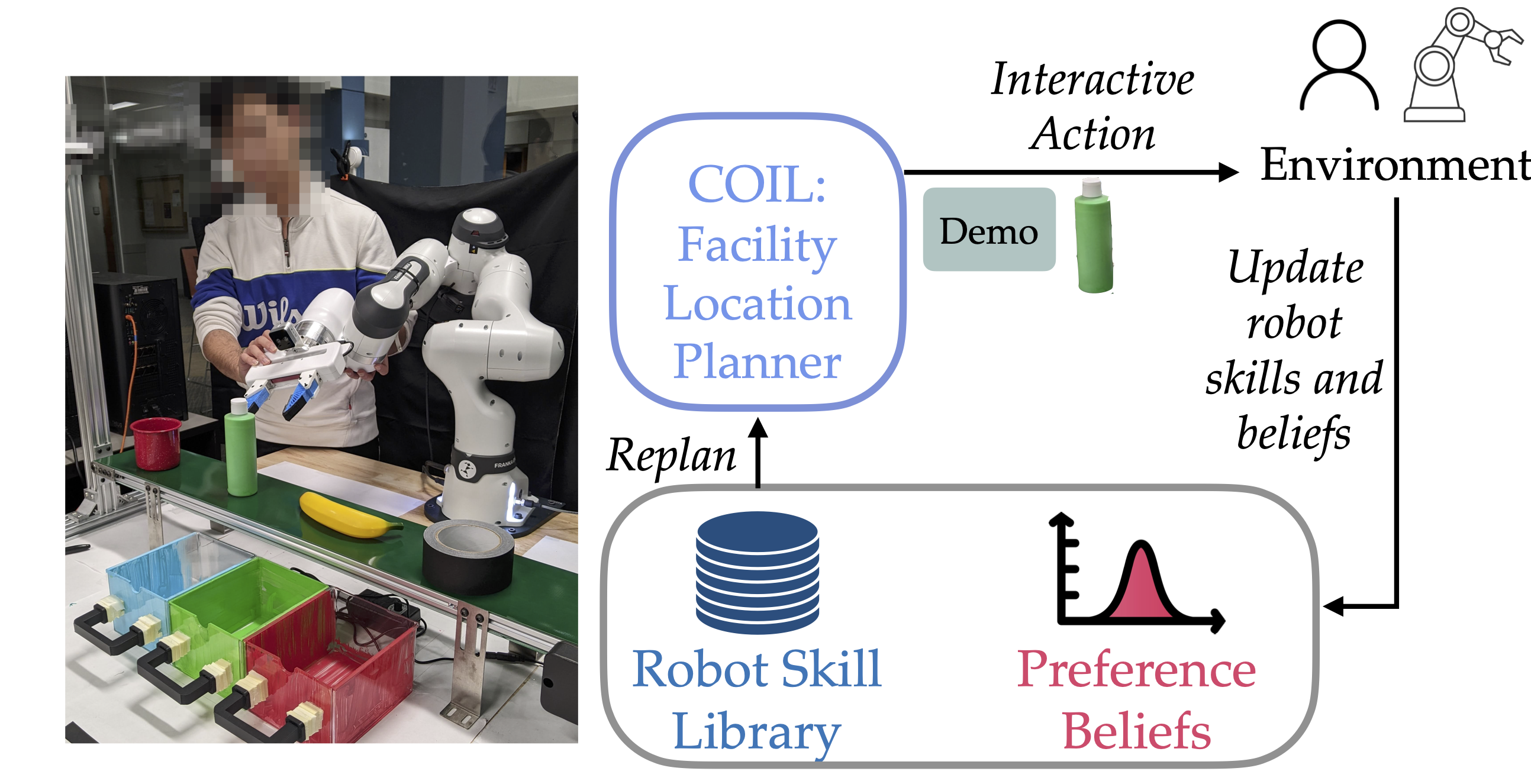}
    \caption{A human-robot team working together to pick objects off a conveyor and pack them in bins. The robot queries the human online to learn motor skills and user preference about how each object should be grasped and where it should be placed. We propose a planning algorithm COIL that optimizes how the robot allocates tasks and uses these different types of queries to minimize human effort during the course of its deployment.}
    \label{fig:teaser}
\end{figure}

We propose a novel multi-task interaction planner \textbf{COIL} (Cost-Optimal Interactive Learning) that uses three types of queries--- \emph{skill}, \emph{preference}, and \emph{human help}---to complete a given sequence of tasks, simultaneously satisfying hidden user preferences and minimizing user burden. 
Our key theoretical contribution is a \textbf{novel formulation} of multi-task interactive robot learning as an instance of the uncapacitated facility location (UFL) problem. Different from prior works, our formulation is cost-optimal and jointly minimizes human burden during learning and deployment.
Our formulation has significant theoretical and practical implications. \emph{Theoretically}, it shows that interactive robot learning in our setting can be approximated, i.e, a bounded sub-optimal solution can be computed in polynomial time. \emph{Practically}, it enables the use strong off-the-shelf UFL solvers for planning. We believe that our UFL formulation opens exciting avenues for future research by leveraging rich theory and algorithms developed by the operations research community, without the need to develop solutions from scratch.
Furthermore, we extend our UFL formulation to incorporate \textbf{multiple query types} by proposing a simple but effective one-step belief space planning algorithm that can request user preferences when uncertain. Notably, even solving this one-step belief space planning problem is NP-hard. Our UFL formulation enables us to leverage fast approximation algorithms as sub-routines to plan efficiently. 
We extensively evaluate and analyze the behavior of our algorithm in simulated and physical manipulation domains. We find that COIL minimizes user burden significantly better than existing approaches, can adapt to failures in teaching, and efficiently scales to long-duration collaboration.  COIL results in interactions with $12\% - 20\%$ and $23\%$ lower cost than the best performing baseline in simulated manipulation experiments and physical conveyor experiments respectively.
Additional details and videos of our experiments are available at \url{https://sites.google.com/view/optimal-interactive-learning}.

\section{Related Work}
\textbf{Learning through multiple interaction modalities} is an active area of robotics research~\cite{hwang2024CVPR, metz2024mappingspacehumanfeedback}. \citet{biyik2022learningfixed} first passively collect demonstrations to initialize a belief about the human's reward function and then actively probe the user with preference queries to zero-in on their true reward. \citet{fitzgeraldINQUIREINteractiveQuerying2022} also aim to learn a human's reward function, but they do so by actively selecting the query from among Corrections, Demonstrations, Preference Queries, and Binary Rewards expected to be most informative to the learner. In a similar vein, \citet{mehtaUnifiedLearningDemonstrations2024} use multiple modalities of physical interaction queries to first learn a reward model online using a neural network representation and then apply constrained optimization to identify an optimal trajectory. Like our work, these works aim to learn from multiple interaction types. However, the aforementioned methods apply to singular tasks, and while they actively query the human for feedback that will be useful to the robot (and can do so while considering the cost of querying), they do not plan for the future utility of learning those tasks in the way our method does.

\textbf{Human preference} is another wide-spread topic of interest in robotics~\cite{biyik2024batch, narcomey2024learning, nemlekar2024pecanpersonalizingrobotbehaviors, peng2024pragmatic, van2024simultaneously}. \citet{jeonRewardrationalImplicitChoice2020} present a unifying formalism for preference-based learning algorithms through multiple types of interaction with the understanding that a ``preference'' is implicit in the feedback provided by the human and the skill being taught. Moreover, these preferences must be learned through multiple interaction instances. In contrast, our work encodes preferences as discrete state or action parameters by which a skill must abide and are learned through just one interaction instance. \citet{hadfield-menellCooperativeInverseReinforcement2016} introduce cooperative inverse reinforcement learning to learn a human's reward function (i.e., preference). While their POMDP-based approach could be applied in our paradigm, our use of a HiP-MDP~\cite{doshi2016hidden} makes the problem tractable and enables us to design an efficient, effective planner. Work like \citet{bajcsyLearningPhysicalHuman2018} isolates a user's preferences by learning features one-at-a-time until converging to the complete reward model, but they do so solely through demonstration queries and without considering future utility. Finally, reinforcement learning from human feedback (RLHF) methods have achieved state of the art results in sample complexity and reward learning for complex tasks~\cite{abdoRobotOrganizeMy2015, christiano2017deep, hejna2023contrastive, rafailov2024direct}. Even so, these successes come with a substantial cost of requiring hundreds or even thousands of preference queries and human-hours, rendering them infeasible for on-the-job learning from people.

\textbf{On-the-job collaborative methods} automate task allocation between machine and human either through modeled or learned approaches~\cite{basichLearningOptimizeAutonomy2020, inagakiAdaptiveAutomationSharing2001, shannon2017human, zhao2023contributionprefs, zhao2024conformaldagger}. For instance, \citet{liuRobotLearningJob2023} propose a learning and deployment framework in which the human monitors the robot's performance to intervene and provide corrections. Data from deployment is then used in subsequent rounds to improve the robot's policy.
\citet{vats2022synergistic} use a mixed integer program to decide when to learn new skills and when to delegate tasks to the user. While their method accounts for the future utility of learning a task, they do so without also considering the human's preferences. \citet{racca2019teacher} argue that active learning methods can increase the cognitive effort of human teachers by asking difficult questions that require frequent context change. This point suggests that sample-efficient learning cannot be the sole end goal; the difficulty of answering each question should be taken into account as well. While some of the aforementioned methods do model different costs associated with querying, our method is the first to do so across a stream of incoming tasks  while considering human preferences as well.


\section{Problem Statement}
A robot is asked to collaborate with a user to jointly complete a sequence of physical tasks $\seq{\tau}{i}{1}{n} = \tau_i,\ldots,\tau_n$, e.g., picking objects from a conveyor and sorting them into bins according to hidden user preferences. Each task is described by a vector $\tau_i \in \mathcal{T}$ and has an associated reward function:

\begin{align}
\begin{split}
r^i_{\text{task}} &= r^i_{\text{safe}} + r^i_{\text{pref}} \\
                  &= -c_\text{skill-fail} \cdot \mathbb{I}[\text{safety violated}] + r^i_{\text{pref}} \\
\end{split}
\end{align}

The \emph{safety reward} $r_{\text{safe}}$  models user-agnostic safety constraints that the robot must not violate (e.g., avoid collisions, do not drop objects). The \emph{preference reward} $r_{\text{pref}}$ measures how closely robot actions match user preferences (e.g., placing a plate on the right shelf, not placing fingers inside a mug). Initially, $r_{\text{pref}}$ is unknown to the robot but can be learned by querying the human.
In addition, the robot starts with a base set of skills $\mathcal{L}_0 :=\{\pi_0\}$ which may not cover all tasks.
Hence, the robot should either request the user to handle such tasks or ask the user to teach it new skills such that it can then complete them autonomously.
Our goal is to complete all the tasks while minimizing the burden placed on the user during the interaction.
The human effort required to respond to each type of robot query is quantified using costs: $c_\text{hum}$ for assigning a task to the human, $c_\text{skill}$ for requesting the human to teach a motor skill, and $c_\text{pref}$ for requesting their preference.
Hence, the overall objective is to maximize $J$:

\begin{align}
    J &= \sum_{i=1}^n{r^i_{\text{task}} - c^i_{\text{query}} - c^i_{\text{rob}}},
\end{align}
where $c^i_\text{query} = c^i_\text{hum} + c^i_\text{skill} + c^i_\text{pref}$ measures the cost of querying the human, and $c^i_{\text{rob}}$ is a fixed cost incurred if the robot undertakes the task.
Note that $\max J \equiv \min -J$. We will use this fact later to convert this into a minimization problem.

\section{Approach}
We model each task as a hidden-parameter MDP (HiP-MDP), where some parameters of the reward function that capture user preferences are uncertain. 
Each task HiP-MDP is $\mathcal{M} := (S, A, \Theta, r_\text{task}(s,a ; z), T, P_{\Theta})$, where $\Theta$ is a discrete set of possible reward parameters and $P_{\Theta}$ is the prior over these parameters. The robot updates its posterior $b^{\theta}$ over the parameters based on its interaction with the human. $z$ is a hidden distribution over reward parameters in $\Theta$ and represents the user's preferences for the given task. The reward for each task $r_\text{task}(s,a ; z)$ is thus a function of the hidden preference distribution. 
We model the choice of acting autonomously and the three types of queries that the robot can make as four distinct actions that comprise the robot's action space $A$:
\begin{enumerate}
\item \emph{Execute Skill.} The robot executes a skill to complete the task. This is modeled by a \emph{rob} action which has a fixed cost $c_\text{rob}$ and a variable task reward $r_\text{task}$ that captures how well the robot completed the task.
\item \emph{Query type 1: Request Skill.} Ask the user to teach a new motor skill for a task via demonstrations. The user's effort to perform this query is captured by the cost $c_\text{skill}$. If the robot fails to learn from the human, then the query is deemed a \emph{teaching failure}.
\item \emph{Query type 2: Request Human.} Ask the user to complete the task. The user effort required to fulfill this request is captured by the cost $c_\text{hum}$.
\item \emph{Query type 3: Request Preference.} Ask the user for their preference about how a task should be done. The user effort is captured by the cost $c_\text{pref}$. Some prior work define ``preference query'' as a comparative choice between two possible options. In our case, the user chooses a preferred option from all available options. We will refer to it as a ``preference request'' to avoid confusion.
\end{enumerate}


\textbf{Robot Skills.} The robot learns a library $\mathcal{L}$ of task-specific skills $\pi(s, \theta)$ that are parameterized by task and user preference parameters.
Define the expected return for executing skill $\pi$ with a specific preference parameter $\theta$ to complete task $\tau_i$ as 
\[R^i(\pi, \theta) \leftarrow \sum_{a \sim \pi(s, \theta)} r^i_\text{task}(s, a \vert \theta).\]
To complete $\tau_i$, the robot must not only select which skill to use, but also which preference parameter to execute said skill with.
Let the robot have skills $\mathcal{L} :=\{\pi_1,\ldots,\pi_k\}$.
Then, the robot must choose a combination of skill and preference parameter $\pi^*, \theta^*$ with the highest return:
\begin{align}
\pi^*, \theta^* 
                = \argmax_{\pi \in \mathcal{L},\theta \in \Theta}
    R^i(\pi, \theta).
\end{align}

\textbf{Skill Return Model.}
The true reward function is unknown to the robot at planning time.
Hence, the robot must use a reward model to predict the expected return from skill execution before learning it.
Let the robot's belief about the user's preference for the current task be $b^{\Theta}$.
We provide the robot with a domain-specific function $\rho_\pi^\text{safe}$ to predict the probability of safe execution. $\rho_\pi^\text{safe}(\tau', \theta')$ predicts the probability of safe execution of a skill $\pi$ when deployed on a task $\tau'$ with parameter $\theta'$.
This function is defined based on similarity between $(\tau', \theta')$ and the conditions $(\tau, \theta)$ that $\pi$ will be learned for: $\rho_\pi^\text{safe}(\tau', \theta') = f(\lVert \tau - \tau'\rVert + \lVert \theta - \theta'\rVert)$. See Appendix \ref{ap_sec:skill_return} for our instantiation in experiments.
Intuitively, the probability of skill success is high if both tasks and preferences are similar and low otherwise.
The preference belief $b^\Theta$ models the probability of preference satisfaction.
The skill return model is then:
\begin{align}
\begin{split}
    \hat{R}^\tau(\pi, \theta') = -[&c_\text{skill-fail} \cdot
    (1 -  \lambda_\text{teach} \cdot \rho^\text{safe}_\pi(\tau', \theta')) + \\
    &c_\text{pref-fail} \cdot (1 - b^\Theta(\theta'))].
\end{split}
\end{align}

The first term predicts the expected penalty for safety violation and the second terms predicts the penalty for preference violation.
$\lambda_\text{teach}$ estimates the probability that the robot will successfully learn a skill from the human.
We model the success of teaching as a Bernoulli process and use Bayesian inference to compute the posterior distribution with a  Beta distribution $\text{Beta}(\alpha, \beta)$ as the prior.
We use the mean of the distribution in the return model as the estimated probability of successful teaching:
\[\lambda_\text{teach} = \mathbb{E}[{\text{Beta}(\alpha, \beta)}] = \frac{\alpha}{\alpha + \beta}.\]
This enables the robot to identify difficult-to-learn tasks and \emph{adapt} its plan online if teaching fails so that it does not waste human time. Our experiments will show that non-adaptive methods expend significant human effort trying to learn such tasks even after repeated teaching failures whereas COIL assigns them to the human and focuses on learning feasible tasks.

\subsection{Skill Learning Under Known Preferences: Facility Location Formulation}
We first consider the case in which the robot cannot query the human about their preference online and must plan with its belief about the preferences. Previous work by \citet{vats2022synergistic} proposed a mixed integer programming formulation ADL to plan in a similar setting. Unfortunately, solving MIP optimally is NP-hard which makes it difficult to scale ADL to larger problems. We propose a novel uncapacitated facility location (UFL) formulation that has tight polynomial-time approximation algorithms which can efficiently compute high-quality solutions even for large problems. Intuitively, the challenge of identifying the minimal cost set of interactive actions that cover the full task sequence maps nicely onto the UFL problem which seeks to service demands (tasks) by allocating facility resources (interactive actions) while minimizing costs. 

\textbf{Facility Location Problem.}
We first briefly introduce the uncapacitated facility location (UFL) problem.
Please refer to \citet[chapter~4]{williamson2011design} for a more detailed treatment.
The facility location problem has a set of demands $D =\{1,\ldots,m\}$ and a set of facilities $F = \{1,\ldots,n\}$.
There is \emph{facility cost} $f_i$ associated with opening each facility $i \in F$ and an \emph{assignment} or \emph{service cost} $c_{ij}$ of serving demand $j$ by facility $i$.
The goal is to serve all the demands by opening a subset of facilities $F' \subseteq F$ such that the overall  cost of opening the facilities in $F'$ and the cost of assigning each demand $j \in D$ to the nearest facility $i \in F'$ is minimized:

\begin{equation}
\min \underbrace{\sum_{i\in F'} f_i}_\text{facility cost}+ \underbrace{\sum_{j\in D} \min_{i\in F'} c_{ij}}_\text{service cost}.
\label{eq:facility_obj}
\end{equation}


We formulate the interaction planning problem as a facility location problem by defining demands, facilities, facility costs and service costs $\langle D, F, f_i, c_{ij} \rangle$.

\begin{figure}[t]
\centering
\includegraphics[width=1\columnwidth]{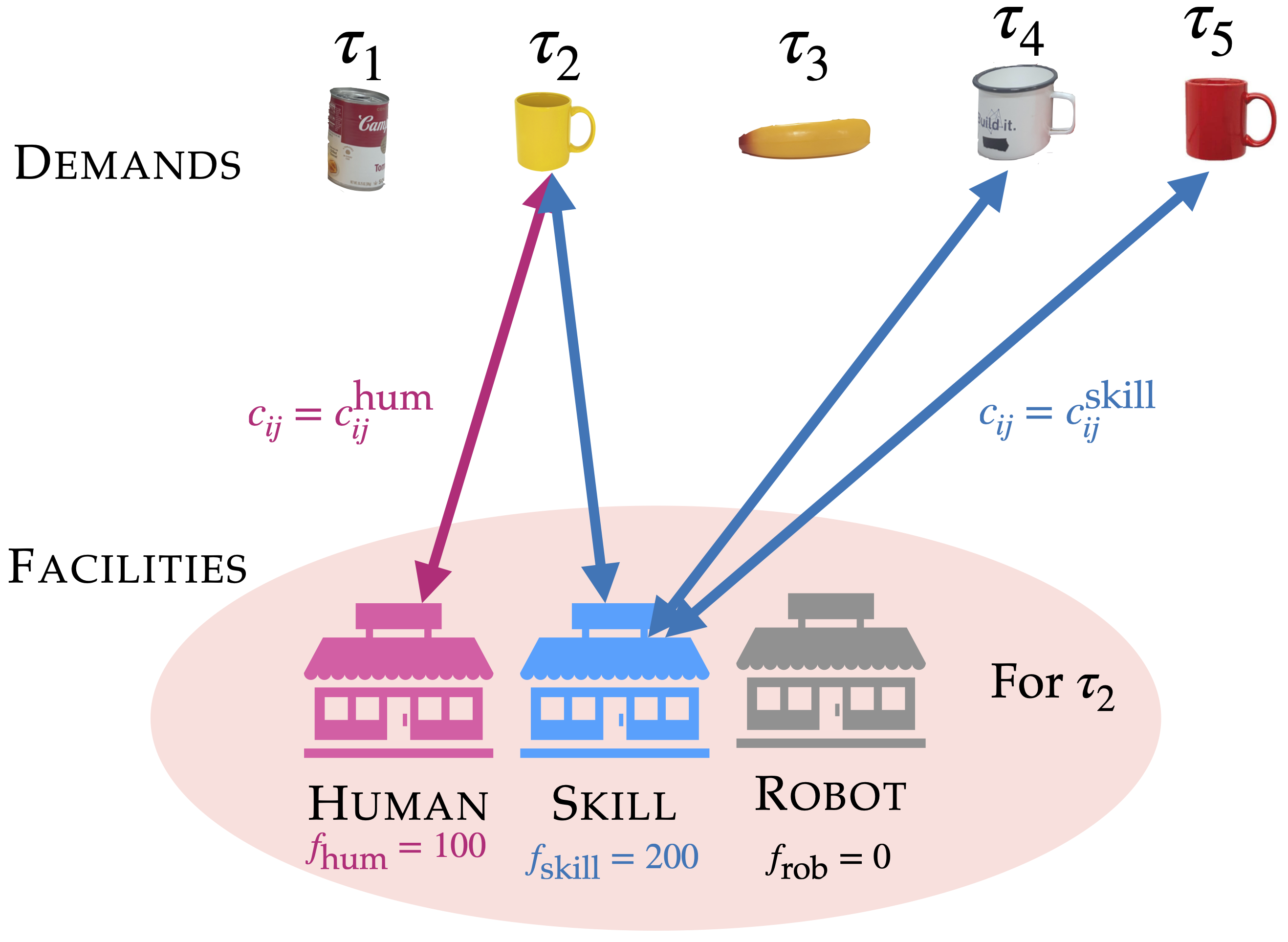}
\caption{Facility location formulation. Tasks $\tau_1,\ldots,\tau_5$ are demands to be satisfied. Facilities correspond to interactive actions available for every task. We highlight facilities for $\tau_2$: \emph{Human} facility can only service $\tau_2$. \emph{Skill} facility can service similar future tasks $\tau_2, \tau_4, \tau_5$. \emph{Robot} facility cannot service any task as the robot hasn't learned a skill yet. Furthermore, none of the facilities can service past tasks.}
\label{fig:ufl_formulation}
\end{figure}

\textbf{Tasks as Demands.} Define the set of demands as the set of tasks $F = \{\tau_1,\ldots,\tau_N\}$. A solution to UFL has to satisfy all the demands. Hence, this guarantees that it will complete all the tasks. 

\textbf{Robot Actions as Facilities.}
In the following, we define facilities corresponding to all the actions for every task $\tau_i$.

 \begin{enumerate}
\item \emph{Human Facilities.} Define one \emph{human} facility $i_\text{hum}$ with a facility cost of $c_\text{hum}$ that captures the cost of asking the user to complete the task. The service cost to serve $\tau_i$ is $0$ as human effort is modeled by $c_\text{hum}$ and  $\infty$ for all the other tasks as this action cannot not complete them.
\item \emph{Skill Facilities.} Define one \emph{skill} facility $i_\text{skill}$ with a facility cost of $c_\text{skill}$ that corresponds to requesting the user to teach the robot a new skill $\pi_i$ for $\tau_i$.
The service cost for completing a task $\tau_j$ consists of a fixed robot execution cost $c_\text{rob}$ and the task reward
\[c^\text{skill}_{ij} = c_\text{rob} - \max_{\theta \in \Theta} \hat{R}^j(\pi_i, \theta).\]
The robot executes $\pi_i$ with preference parameters $\theta^*$ that maximize the expected reward under its preference belief distribution.

\item \emph{Robot Facilities.} Define one \emph{robot} facility $i_\text{rob}$ with a facility cost of $0$ that corresponds to the robot using a previously learned skill $\pi_i$. The service cost for completing a task $\tau_j$ consists of a fixed robot execution cost $c_\text{rob}$ and the task reward as above.
\end{enumerate}


\begin{algorithm}[t]
\caption{Facility location formulation for COIL.}
\label{alg:adl}
    \begin{algorithmic}[1]
    \Procedure{COIL-KnownPrefs}{$\seq{b}{i}{1}{N}$}
        \State{$D \gets \{\tau_1,\ldots,\tau_N\}, F \gets \emptyset$ \Comment{demands and facilities}}
        \For{$i \in 1,\cdots,N$}
            \State{$F.\text{insert}(i_\text{hum})$, $f_i^\text{hum} \gets c_\text{hum}$} \Comment{human facility}
            \State{$c_{ij}^\text{hum} \gets 0$ if $i = j$ else $\infty$} 
            \State{$F.\text{insert}(i_\text{skill}), f_i^\text{skill} \gets c_\text{skill}$} \Comment{skill facility}
            \For{$j \in {i,\ldots,N}$}
                \State{$c^{\text{skill}}_{ij} = c_\text{rob} - \max_{\theta \in \Theta} \hat{R}^j(\pi_i, \theta)$}
            \EndFor
        \EndFor
        \For{$i \in \mathcal{L}$} \Comment{skills already learned}
            \State{$F.\text{insert}(i_\text{robot}), f_i^\text{robot} \gets 0$} \Comment{robot facility}
            \For{$j \in {1,\ldots,N}$}
                \State{$c^{\text{robot}}_{ij} = c_\text{rob} - \max_{\theta \in \Theta} \hat{R}^j(\pi_i, \theta)$}
            \EndFor
        \EndFor
        \State{$\seq{a}{i}{1}{N}, J \gets \textsc{solveUFL}(D, F, f, c)$}
        \State{}\Return{$\seq{a}{i}{1}{N}, J$}
    \EndProcedure
    \Procedure{COIL}{$\seq{b}{i}{1}{N}$}
        \State{$\seq{a}{i}{1}{N}, J \gets \textsc{COIL-KnownPrefs}(\seq{b}{i}{1}{N})$}
        \State{$\bar{J} \gets 0$}
        \For{$\theta \in \Theta$} \Comment{possible query response} \label{line:loop_over_prefs}
            \State{$\seq{b'}{i}{1}{N} \gets \textsc{UpdatePrefBeliefs}(\theta)$}
            \State{$J' \gets \textsc{COIL-KnownPrefs}(\seq{b'}{i}{1}{N})$}
            \State{$\bar{J} = \bar{J} + b_1(\theta) \cdot J'$}
        \EndFor
            \If{$c_\text{pref} + \frac{\bar{J}}{\sum b_1} \leq J$} \label{line:pref_cost_check}
            \State{$a_1 \gets a_\text{pref}$}
            \EndIf
        \Return{$\seq{a}{i}{1}{N}$}
    \EndProcedure
    \Procedure{Interaction}{$\seq{\tau}{i}{1}{N}$}
    \State{$\seq{b}{i}{1}{N} \leftarrow \textsc{initPrefBelief}()$}
    \State{$k \leftarrow 1$} \Comment{current task index}
    \While{all tasks are not done}
        \State{$\text{plan} \leftarrow \textsc{COIL}(\seq{b}{i}{k}{N})$} \Comment{replan}
        \State{execute the first action and update $k$}
        \State{update $\lambda_\text{teach}$}
    \EndWhile
    \EndProcedure
    \Return
    \end{algorithmic}
\end{algorithm}

\textbf{Approximation Algorithms.}
We implement the approximation algorithm proposed by \citet{jain1999primal}. In the worst case, this algorithm is 3-suboptimal when the facility location problem is metric and $\log(n)$-suboptimal otherwise.
This algorithm has a run-time of $O(n^2\log(n))$, where $n$ is the number of demands.
In practice, we found it to be near-optimal for our problems.
Let $X$ be the set of facilities opened so far, and let $U$ be the set of demands that are not served by open facilities. Then, the algorithm iteratively picks a demand $i \in F$ and $Y \subseteq U$ that minimizes the ratio $\frac{f_i + \sum_{j\in Y}c_{ij}}{\lvert Y \rvert}$ and sets $f_i$ to $0$. Intuitively, it opens a facility that has the minimum cost per demand for some subset of demands and assigns those demands to the opened facility.
\subsection{Planning for Preference Requests}
Equipped with a provably bounded-suboptimal plan for learning under known preferences as a reference point, COIL then determines when preference requests are needed to clarify user preferences before execution. For each task, COIL evaluates the expected change, under its current beliefs $b^{\theta}$, in overall plan cost should it request the user's preference for the current task. If the expected plan cost plus preference request cost is lower than the current plan (line~\ref{line:pref_cost_check}), COIL elects to first reduce uncertainty in the preference parameters for the current task. In this way, COIL augments its bounded-suboptimal reference plan  with uncertainty-guided preference requests when necessary, towards improving the plan to handle uncertainty in preference parameters.
During the interaction, we replan after every interactive action execution to take into account the latest information. This is enabled by the fast run-time of COIL which is polynomial in the number of tasks and preference parameters.

\begin{theorem}
The worst-case runtime of COIL is polynomial in $n$ and $k$, where $n$ is the number of tasks and $k = \lvert \Theta \rvert$ is the number of possible preference parameters.
\end{theorem}
\begin{proof}
    COIL solves the facility location problem $k = \lvert \Theta \rvert + 1$ times (see line~\ref{line:loop_over_prefs}) in total by calling the function $\textsc{COIL-KnownPrefs}$. The complexity of $\textsc{COIL-KnownPrefs}$ is polynomial when $\textsc{solveUFL}$ is an approximation algorithm.  In particular, we use an approximation algorithm with runtime $O(n^2\log(n)$, where $n$ is the number of tasks. Hence, the overall complexity of COIL is $O(k(n^2\log(n))$, which is polynomial in $n$ and $k$.
\end{proof}

\section{Experimental Setup} \label{sec:experiments} 
To evaluate the efficacy of our proposed approach, we run a series of quantitative experiments to study how \coil affects the costs of learning on the job during task execution compared with baselines. We further investigate quantitatively the nature of queries throughout the interaction to understand how \coil and other approaches induces different interactive behaviors. 

\para{Domains} We study the combined learning and execution costs incurred by our approach through experiments in three controlled environments: an object pickup and dropoff \textbf{Gridworld} environment implemented using MiniGrid~\citep{MinigridMiniworld23}, a simulated \textbf{7DoF Manipulation} environment implemented using robosuite~\citep{robosuite2020}, and a real-world \textbf{Conveyor} manipulation setting. Brief descriptions of each domain are below; please see the appendix \ref{ap_sec:envs} for details.
\begin{enumerate}
\item \textbf{Gridworld} is a discrete, 17x17 grid comprised of a sequence of 15 total objects of nine varieties distinguished by object \emph{type}, \emph{color}, and \emph{position}. Each object defines a task the agent must execute, and each task is to navigate to an object, pick it up, and transport it to the user's preferred location for that object. There are three possible goal locations.
\item \textbf{7DoF Manipulation} is a simulated bin-packing task where a Franka robot manipulator must pick up and put 30 total objects of seven different varieties (e.g., milk carton, mug) into one of four different bins. Each object defines a task the agent must execute.
\item \textbf{Conveyor} is a real-world instantiation of a factory setting in which a Franka tabletop manipulator and a human employee work side by side to sort objects arriving on a conveyor belt into three containers. Some of the objects, such as the mug, have two possible grasps: handle and rim. The robot must pick each object using the preferred grasp and place it in the preferred bin. We randomly sample 5 task sequences consisting of 20 tasks each. We pretrain GEM~\cite{jia2024open} using human demonstrations for every task and use it to provide demonstrations in our experiments.
\end{enumerate}

\para{Baselines} We compare \coil with state-of-the-art approaches for interactive robot learning and task allocation. Implementation details for each baseline can be found in the Appendix \ref{sec:baselines}.
\begin{enumerate}
    \item \emph{Confidence-based ADL (\cadl).} ADL~\citep{vats2022synergistic} addresses the similar problem of planning for skill learning and task allocation and thus makes for a strong baseline. However, ADL does not aim to learn human preferences and thus lacks the ability to reason jointly over preference and skill learning. To ensure a fair comparison, the \cadl baseline makes preference queries when the robot's confidence over the human's preference is below some predefined threshold.
    
    \item \emph{Information Gain (\info).} Because \coil plans over multiple query types (skill, preference) and multiple human contribution types (human, robot), it makes sense to design a baseline inspired by interactive learning methods that incorporate multiple types of human feedback (e.g., demonstrations, preferences, etc.). In these methods, a widely-used approach is to select the query which provides the greatest expected information gain \cite{fitzgeraldINQUIREINteractiveQuerying2022}. While these methods typically learn reward functions for a single task and consider neither human contributions nor the future utility of learning particular skills, we design an information gain objective which does both as the \info baseline in our interaction paradigm.
    
    \item \emph{Confidence-based Autonomy (\cba).} Inspired by~\citet{chernovaConfidencebasedPolicyLearning2007}, \cba requests skill and preference teaching if the robot is uncertain about user preferences or skills and otherwise assigns the task either to itself or the human.
    
\end{enumerate}

\para{Human Cost Profiles}
The costs a human might associate with teaching, executing tasks, or responding to preference requests depend on domain-specific variables such as the cost of labor and the difficulty of teaching. We study the performance of our approach under multiple simulated users with three different cost profiles. In our evaluations, we assign a cost of $c_\text{rob}=10$ to each robot execution, $c_\text{hum}=80$ to human task execution, and $c_\text{pref}=20$ to each preference request. The robot incurs a penalty cost of $c_\text{skill-fail}=100$ if it fails to successfully execute any skill, and a penalty of $c_\text{pref-fail}=100$ if it successfully executes a skill by placing the object in a goal undesired by the user. 
To compare the profiles with minimal cost tuning, we examine profiles across the cost of human skill teaching.
\begin{enumerate}
    \item \emph{Low-Cost Teaching} ($c_\text{skill}=50$): This profile simulates an experienced teacher for whom providing demonstrations of robot skills is not burdensome. 
    \item \emph{Med-Cost Teaching} ($c_\text{skill}=100$): This profile simulates a teacher for whom teaching is moderately more burdensome than performing the task themselves.
    \item \emph{High-Cost Teaching} ($c_\text{skill}=200$): This profile simulates a novice teacher for whom providing demonstrations of robot skills is highly burdensome. 
\end{enumerate}

\textbf{Preference Belief Estimation.} The robot maintains a belief estimate $b^{\Theta}$ of the user's hidden personal preferences for every task. The probability that each preference parameter $\theta$ satisfies user preference, i.e., $b^{\Theta}(\theta)$, is modeled as a Bernoulli distribution. These belief estimates are updated based on feedback from the user using a Bayesian filter (described in Appendix \secref{sec:pref_filter}).

\section{Experimental Results} \label{sec:results} 
We summarize our results through four major takeaways derived from five distinct experimental focuses: (A) cost comparisons under different human cost profiles when all task skills are learnable by the robot, (B) scalability and computational costs of \coil under long task sequences, (C) online-adaptiveness of \coil when challenging-to-learn skills necessitate online replanning, and (D) a real-world instantiation of \coil that enables the human-robot team to teach and learn as they manipulate objects which arrive on a conveyor.


\begin{figure}[!tb]
    \centering
    \includegraphics[width=0.99\columnwidth]{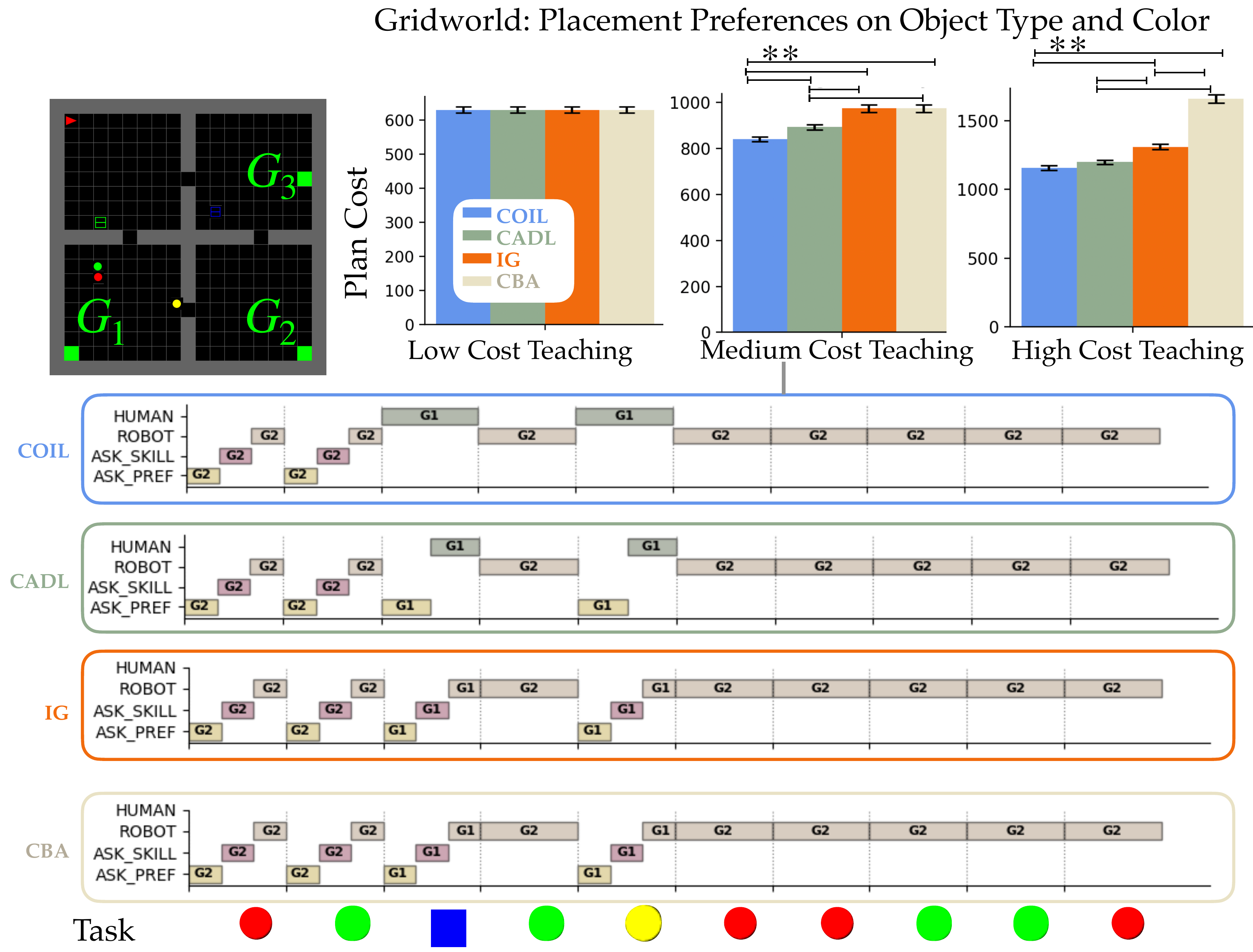}
    \caption{Under Med- and High-Cost teaching cost profiles, \coil consistently chooses the lowest cost plan compared with baselines. We highlight the qualitative behavior of \coil compared to baselines: \coil under the medium cost teaching profile assigns singleton tasks to the human when the cost of learning is high. Error bars indicate standard error over 30 randomized task sequences and true human preferences in the Gridworld domain. ** represents $p<0.01$ significance.}
    \label{fig:gridworld_res}
\end{figure}

\begin{table}[]
\centering
\footnotesize
\resizebox{0.8\textwidth}{!}{\begin{minipage}{\textwidth}
\begin{tabular}{@{}lllllll@{}}
\toprule
    Cost  & Algo & \#teach & \#human & \#pref                    & \#robot &  Cost\\ 
    \midrule
Low  & COIL &  4.43 (0.88)       & 2.43 (1.26)     & 4.43 (0.88) & 12.57 (1.26)   & \textbf{630.67} (50.26)     \\
& C-ADL  & 6.87 (0.72)      & 0.0 (0.0)       & 6.87 (0.72) & 15.0 (0.0)  & 630.67 (50.26)     \\
& IG  & 6.87 (0.72)      & 0.0 (0.0)      & 6.87 (0.72) & 15.0 (0.0)   & 630.67 (50.26)   \\
& CBA  & 6.87 (0.72)      & 0.0 (0.0)      & 6.87 (0.72) & 15.0 (0.0)   & 630.67 (50.26)   \\

\midrule
Med  & COIL & 4.13 (0.85)     & 2.77 (1.26)      & 4.13 (0.85) & 12.23 (1.26)   & \textbf{839.67} (53.7)   \\
& C-ADL  & 4.17 (0.86)       & 2.7 (1.29)      & 6.87 (0.72) & 12.3 (1.29)  & 893.0 (60.84)     \\
& IG  & 6.87 (0.72)   & 0.0 (0.0)       & 6.87 (0.72) & 15.0 (0.0)   & 974.0 (86.16)   \\
& CBA  & 6.87 (0.72)   & 0.0 (0.0)       & 6.87 (0.72) & 15.0 (0.0)   & 974.0 (86.16)   \\

\midrule
High  & COIL & 0.37 (0.66)       & 13.2 (3.12)      & 0.53 (0.67) &  1.8 (3.12)   & \textbf{1158.0} (85.42)  \\
& C-ADL  & 2.33 (0.7)      & 6.37 (1.72)   & 6.87 (0.72) & 8.63 (1.72)  & 1199.67 (78.42)   \\
& IG  &  4.17 (0.86)       & 2.7 (1.29)     & 6.87 (0.72)) & 12.3 (1.29) & 1309.67 (108.27)  \\
& CBA  & 6.87 (0.72)    & 0.0 (0.0)     & 6.87 (0.72) & 15.0 (0.0)  & 1660.67 (157.96)    \\ 

\bottomrule
\end{tabular}
\end{minipage} }
\caption{In the Gridworld domain, we find that COIL makes fewer preference queries than the confidence-based baselines because COIL only asks for human preferences if it believes that this information will be useful later. Format is mean(standard deviation).}
\label{tab:grid_compare_num_prefs}

\end{table}

\begin{table}[]
\centering
\footnotesize
\resizebox{0.75\textwidth}{!}{\begin{minipage}{\textwidth}
\begin{tabular}{@{}lllllll@{}}
\toprule
    Cost  & Algo & \#teach & \#human & \#pref                    & \#robot &  Cost\\ 
    \midrule
Low  & COIL    & 6.6 (1.50) & 4.0 (3.58) & 5.6 (1.02) & 19.4 (4.29) & \textbf{1519.0} (351.86) \\
     & COIL-NoAd & 8.3 (2.9) & 0.8 (0.98) & 5.6 (1.02) & 20.9 (2.84) & 1630.0 (402.34) \\
     & C-ADL  & 8.8 (3.76) & 0.0 (0.0) & 6.4 (0.8) & 21.2 (3.76) & 1705.0 (509.91)  \\
     & IG  &  30.0 (0.0) & 0.0 (0.0) & 6.4 (0.8) & 0.0 (0.0) & 2403.0 (279.11) \\

\midrule
Med  & COIL          & 6.0 (1.095)      & 5.1 (3.91)        & 5.6 (1.0198) & 18.9 (4.4821)   & \textbf{1722.0} (406.98)    \\
     & COIL-NoAd & 7.6 (2.61)      & 1.5 (0.92) & 5.6 (1.02) & 20.9 (2.84) & 1940.0 (474.15)    \\
     & C-ADL         & 7.4 (2.97)       & 1.4 (1.2)       & 6.4 (0.8) & 21.2 (3.76)  & 1990.0 (638.12)     \\
     & IG            & 30.0 (0.0)      & 0.0 (0.0)       & 6.4 (0.8) &  0.0 (0.0)   & 3839.0 (274.02)   \\

\midrule
High    & COIL & 4.1 (0.3) & 8.0 (4.54) & 4.0 (0.0) & 17.9 (4.72) & \textbf{2099.0} (391.80) \\
        & COIL-NoAd & 5.9 (2.34) & 5.3 (2.53) & 4.0 (0.0) & 18.8 (3.79) & 2495.0 (785.94) \\
        & C-ADL  &  6.3 (2.93) & 3.0 (1.84) & 6.4 (0.8) & 20.7 (3.80) & 2644.0 (864.56) \\
        & IG  & 0.0  (0.0) & 30.0 (0.0) & 6.4 (0.8) & 0.0 (0.0) & 2528.0 (16.0) \\

\bottomrule
\end{tabular}
\end{minipage} }
\caption{Results on the manipulation domain. On average, \coil plans interactions that result in $7\%$ to $18\%$ reduction in cost compared to the best performing baseline. The improvement over baselines is particularly marked when the cost of teaching is more expensive than assigning the task to the human, i.e, medium and high cost profiles. The reported statistics are averaged over 10 interactions with 30 randomly sampled tasks each.
}
\label{tab:manip_compare_num_prefs}
\end{table}

\subsection{COIL outperforms myopic interactive learning.}
\coil significantly outperforms myopic interactive learning methods that do not consider plan over the future utility of all possible plans (Figure \ref{fig:gridworld_res}). In the Gridworld domain instantiated such that all objects are learnable by the robot, we evaluate the incurred plan costs over 30 task sequences and randomized human object arrangement preferences. Under the low-cost profile, \coil and all baselines achieve similar plan costs (no significance). As the cost of human demonstrations increases (med-cost and high-cost), \coil consistently achieves the lowest-cost plans of all the methods. We compared statistical differences using a one-way ANOVA \cite{st1989analysis} between costs achieved for each algorithm, and evaluated pairwise differences using pairwise t-tests \cite{kim2015t} with Bonferroni correction \cite{weisstein2004bonferroni} if significant main effects were present. For the Med-Cost human cost profile in the Gridworld domain, we found significant effects of planner ($F=23.48, p<0.01$), and found significant ($p<0.01$) pairwise differences between \coil and all baseline algorithms. For the High-Cost human cost profile in the Gridworld domain, we found significant effects of planner ($F=120.58, p<0.01$), and found significant ($p<0.01$) pairwise differences between \coil versus \info and \cba. Though we did not find a pairwise significant difference, we observe that \coil achieves a slightly lower cost on average compared with \cadl (Figure \ref{fig:gridworld_res}). We provide all statistical testing values in Appendix \ref{ap_sec:statistical}.

To understand the differences between the behaviors induced by each algorithm, the bottom four rows of Figure \ref{fig:gridworld_res} highlight a task sequence of 10 objects when each method collaborates with the med-cost human profile. \coil identifies the lowest-cost plan by requesting to be taught the skills and preferences associated with repeated instances of the same object and delegating singleton tasks to the human partner. \cadl identifies a plan similar to \coil, but it asks for preference queries for singleton tasks that are eventually assigned to the human, incurring additional preference cost. Under Med-Cost teaching, \info and \cba opt to learn every preference and skill (Table \ref{tab:grid_compare_num_prefs}). On the other hand, under Low-Cost teaching, all approaches choose to learn all preferences and skills, given that the total cost of doing so is not expensive. These results highlight a particular strength of \coil to balance learning, requests for human contribution, and execution especially when the right course of action under nuanced costs is not easily known.

\subsection{COIL efficiently scales to long task sequences without compromising plan quality.}
We compare the approximation algorithm used in COIL with an optimal mixed-integer programming (MIP) approach.The MIP formulation is implemented using Gurobi~\citep{gurobi}, a state-of-the-art MIP solver, while our algorithm is implemented in Python, leaving room for further improving run-time.
COIL efficiently computes near-optimal solutions significantly faster than MIP. The speedup is defined as the ratio of MIP runtime to COIL runtime, while sub-optimality is measured as the ratio of the cost of plans generated by COIL to those produced by MIP. In Figure \ref{fig:suboptimality}, we observe that the computational advantage of COIL over MIP gets more pronounced for longer task sequences without compromising on solution quality.


\begin{SCfigure}[50][!bt]
    \centering
    \includegraphics[width=0.36\textwidth]{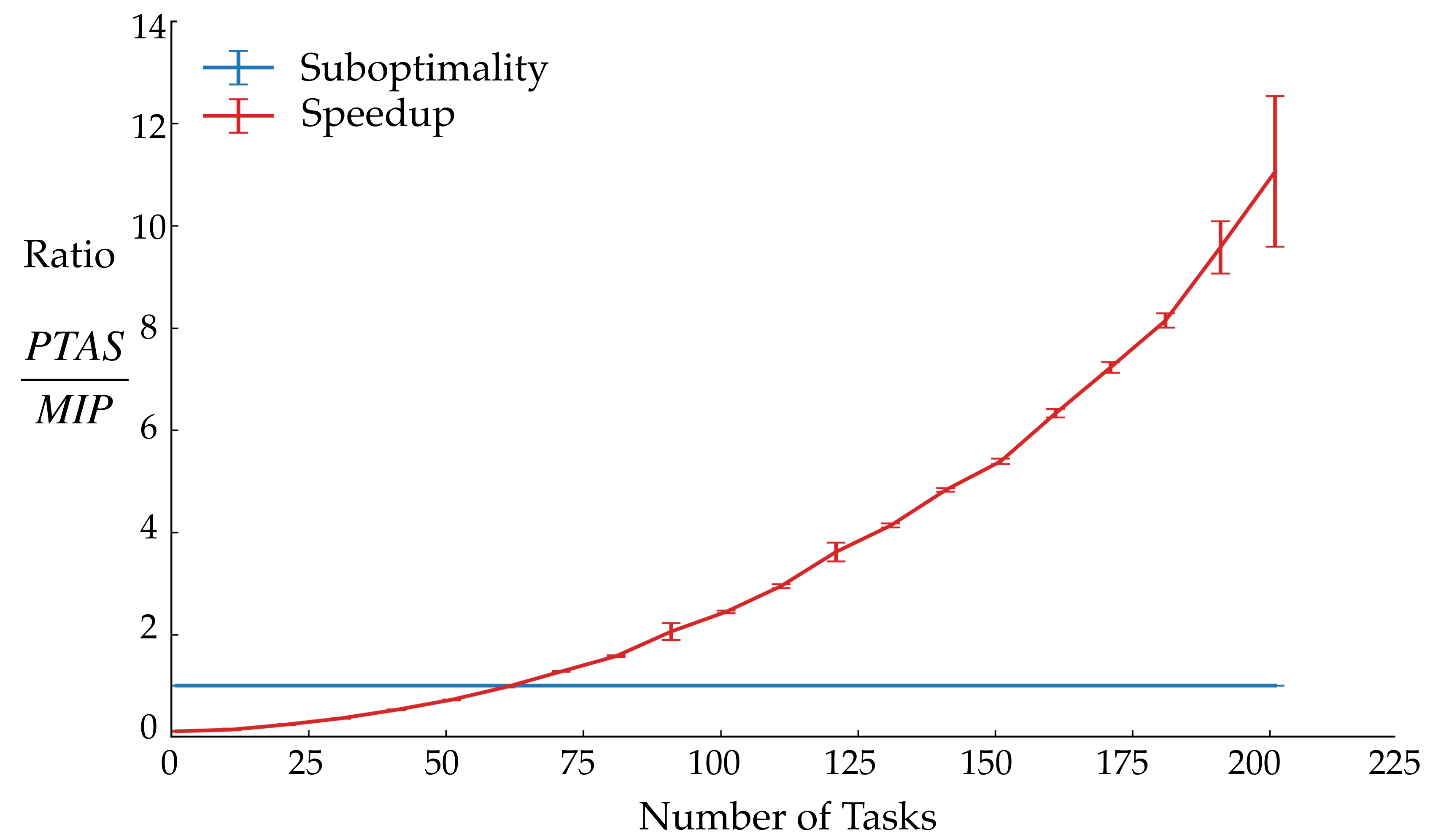}
    \caption{COIL computes near-optimal solutions significantly faster than optimal mixed integer programming.}
    \label{fig:suboptimality}
\end{SCfigure}

\subsection{Adapting to teaching failures online reduces human burden.}
Thus far, our assumption that the robot can reliably learn skills for all tasks enabled us to evaluate and compare the plan costs incurred by \coil and baselines.
While accurately modeling the feasibility of skill learning is not the focus of this work, we investigate how \coil can adapt to failures in learning with online replanning.
Importantly, these ``challenging'' skills are comparatively difficult to learn and are revealed as such only after the robot tries to acquire the skill via human demonstrations. We demonstrate these results in both the Gridworld and Simulated 7DoF Manipulation domains.

\subsubsection{COIL Adaptivity in Gridworld}
In the simulated Gridworld domain, we control for the presence of challenging skills by varying the number of skills that are challenging-to-learn (i.e., can never be learned), and examine the relationship between number of challenging skills and the benefit of estimating $\lambda_\text{teach}$ online in the \coil framework. As the proportion of challenging skills increases (from 10\% to 50\% (Fig \ref{fig:gridworld_unteachable50}) to 90\%), the added benefit of replanning based on the observed feasibility of teaching enables \coil to identify lower-cost plans than \coilnoadapt which optimistically assumes that all skills are learnable. Baselines that do not adapt to teaching failures also incur higher costs than \coil. (Fig \ref{fig:gridworld_unteachable50}). Refer to Appendix \ref{ap_sec:failures} for reported statistical analyses and detailed analysis of results on 10\% and 90\% challenging-to-objects.

\begin{figure}[!tb]
    \centering
    \includegraphics[width=0.99\columnwidth]{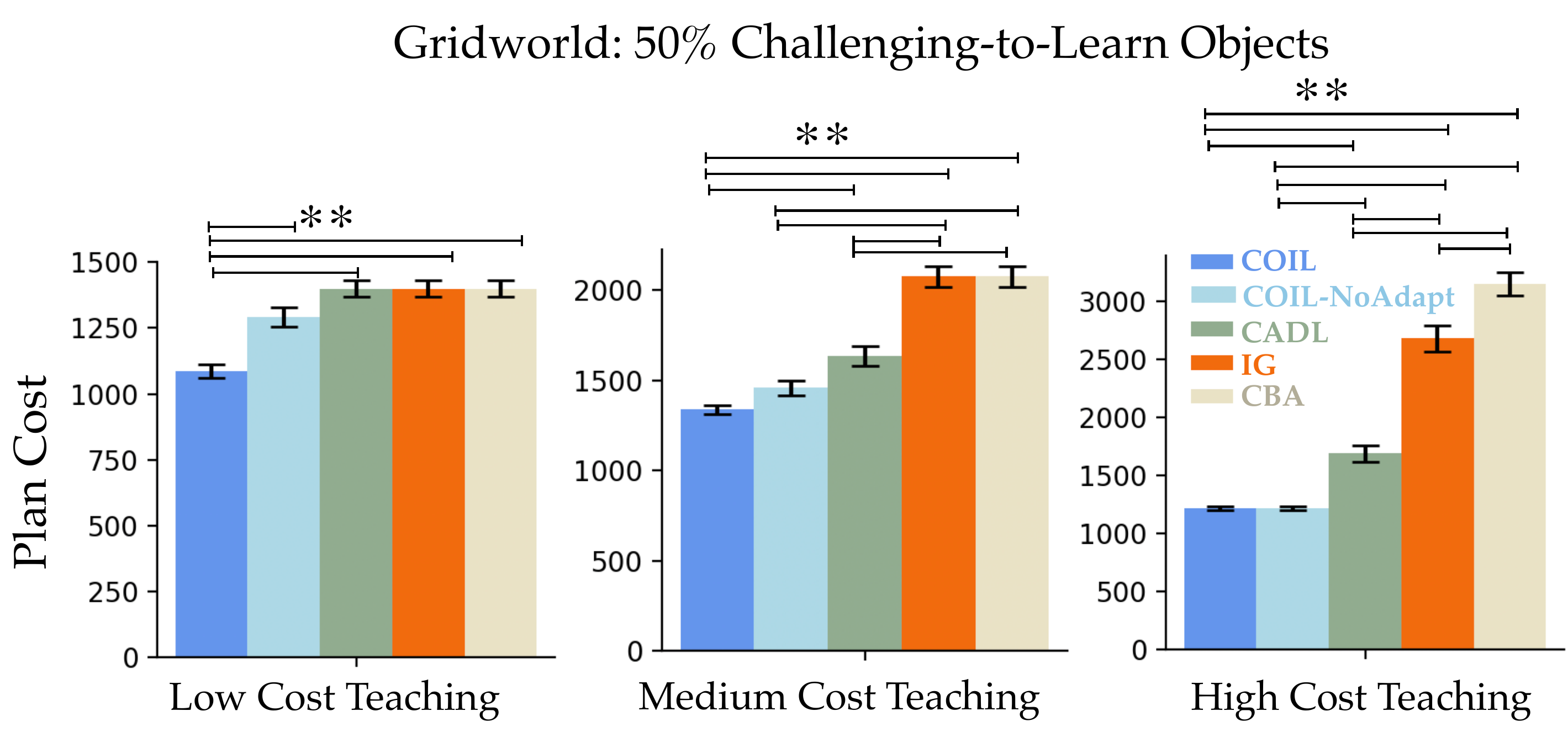}
    \caption{In the Gridworld environment, when half of the objects are challenging-to-learn, \coil achieves the lowest plans, compared to \coilnoadapt and other baselines across teaching profiles. In high cost teaching, \coil and \coilnoadapt often assign tasks to the human off the bat, reducing the impact of adaptivity. Error bars indicate standard error over 30 randomized task sequences and true human preferences. ** represents $p<0.01$ significance.}
    \label{fig:gridworld_unteachable50}
\end{figure}

\subsubsection{COIL Adaptivity in Simulated 7DoF Manipulation}
We next study the adaptivity of \coil in the simulation manipulation domain.
We consider task sequences of 30 objects, sampled from 7 unique varieties. The mug object is too wide for the robot gripper which results in a teaching failure when the human tries to teach the robot.
\coil takes this failure into account by updating $\lambda_\text{teach}$ for all mugs and adapts its plan to assign mugs to the human. This results in statistically significantly lower costs than baselines (figure \ref{fig:manip_sim_domain}, table \ref{tab:manip_compare_num_prefs}). In this experiment, we compare with our 3 strongest baselines, derived from our earlier results, removing \cba from our analysis. Statistical analyses are detailed in Appendix \ref{ap_sec:manip_statistical}.
\begin{figure}[!tb]
    \centering
    \includegraphics[width=0.99\columnwidth]{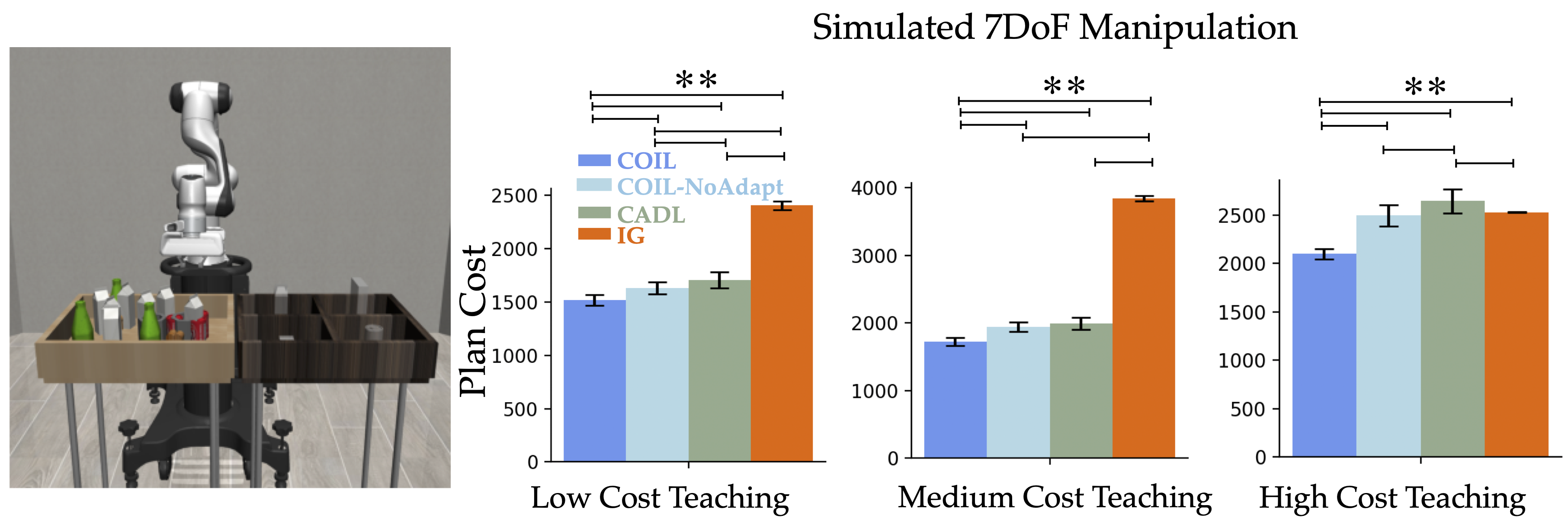}
    \caption{In the Simulated 7DoF Manipulation domain, \coil achieves significantly lower plan costs than baseline methods. Error bars represent standard error over 10 randomized initialization of the 30-object task sequence.}
    \label{fig:manip_sim_domain}
\end{figure}

\subsection{COIL succeeds in real-world operations.}

\begin{table}[!bt]
\footnotesize
\centering
\footnotesize
\resizebox{0.8\textwidth}{!}{\begin{minipage}{\textwidth}
\begin{tabular}{@{}lllllll@{}}
\toprule
     Algorithm  & \#teach & \#human & \#pref & \#robot &  Cost \\ 
    \midrule
 COIL           & 1 (1)       & 12.25 (4.65)   &  1.5 (1)  & 6.25 (3.77)   & \textbf{870} (94.52)     \\
 CADL          & 2 (1.15)       & 9.75 (3.86)  & 14 (1.63)  & 8.25 (2.75)    & 1075 (59.72)      \\
    \midrule
 COIL (teach fail)           & 1       & 19       & 1 & 0   & \textbf{1280}     \\
 CADL (teach fail)          & 6       & 10       & 16 & 4  & 2320    
\end{tabular}
\end{minipage}}
\caption{Results on a physical conveyor. We ran experiments with 5 different task sequences, each with 20 objects, with COIL and CADL. We observed teaching failure on the white mug (possibly because its shiny surface made camera-based pose estimation difficult). Hence, we report the run with teaching failure separately from the other 4 runs. COIL was able to achieve significantly lower cost than the baseline in both situations. CADL especially struggled in the case of teaching failure as it repeatedly requested to be taught the mug skill.}
\label{tab:conveyor}
\end{table}

We evaluate \coil on a physical conveyor domain designed to emulate collaboration in factories. We ran the experiments with a medium cost profile with $c_\text{hum} = 50, c_\text{skill} = 100$, 5 different task sequences, each with 20 objects randomly sampled from a set of 12 objects. Each sequence consisted of one high-frequency object which was five times more likely to be sampled than the other objects. The objects appeared one-by-one in front of the robot, at which point it needed to decide whether to autonomously pick up the object using an appropriate grasp and place it in the correct bin, or request help from the human. Some objects---such as mugs and bottles---had two feasible grasps, of which only one was preferred by the human. Similarly, the human had hidden preferences about which bin each object belonged to. The robot queried the human to understand their preferred grasp and target bin for each object.
\begin{figure}[!ht]
\centering
\includegraphics[width=1.\columnwidth]{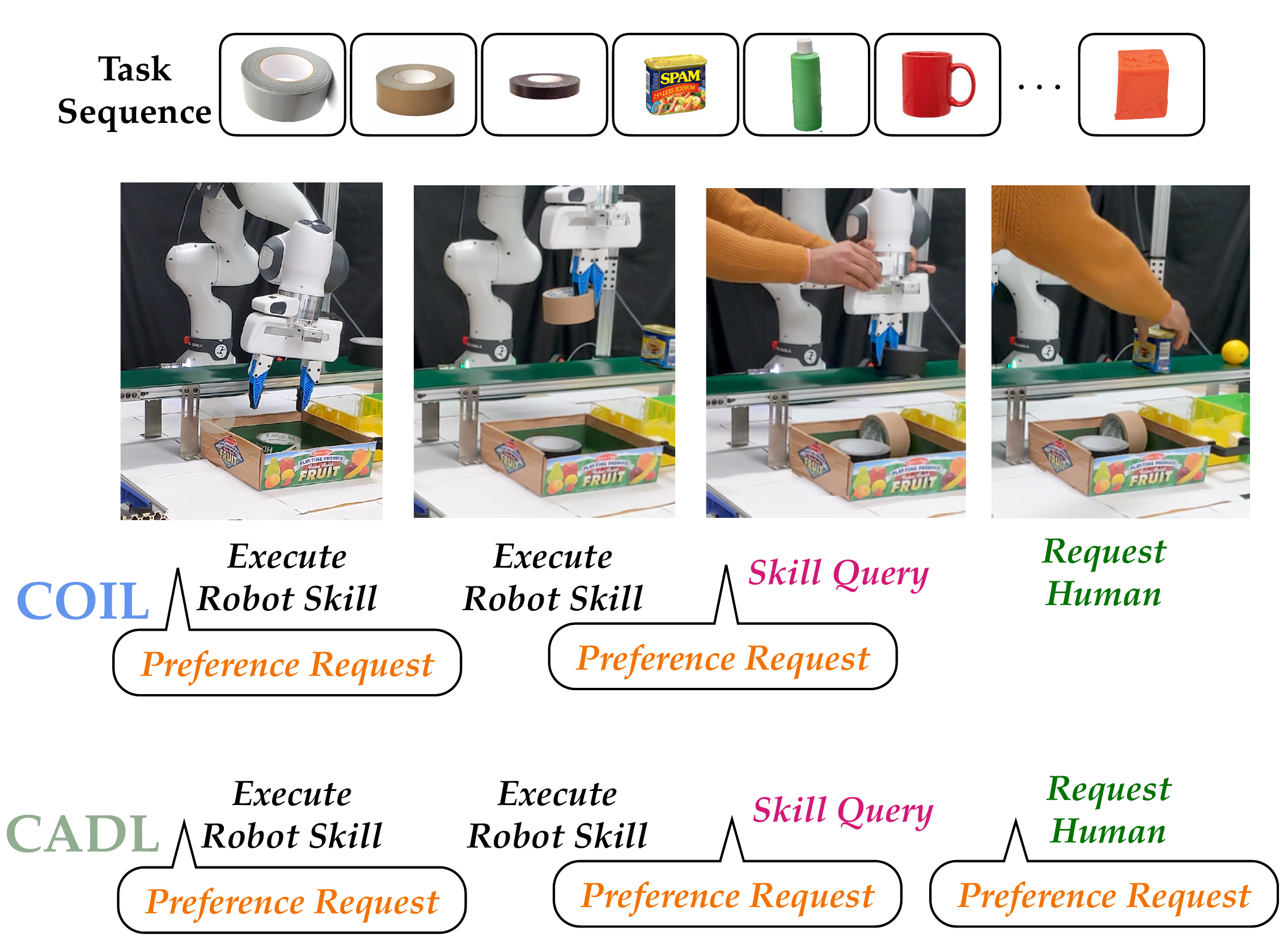}
\label{fig:real_objs}
\caption{In the real-world conveyor task sequence, \coil minimizes costs relative to \cadl, requesting user preferences when uncertain, executing known skills, and requesting demonstrations to learn new skills. We highlight three tasks along the plan to learn generated by \coil. \cadl requests unneeded preferences before assigning tasks to the human.}
\end{figure}

We report the results from our experiment in table \ref{tab:conveyor}. Compared to \cadl, \coil made fewer skill and preference requests. \cadl has a tendency to overburden the user with preference requests while \coil learns user preference for a task only when it intends to complete the task autonomously. A priori, it is  desirable to learn how to manipulate the high-frequency object in each experiment. However, we observed that the robot was unable to learn a skill for the white mug, possibly because its shiny surface made our camera-based pose estimation difficult. Because of its adaptive nature, \coil changed its plan and assigned all mugs to the human. By contrast, \cadl repeatedly requested the human to teach it how to manipulate the mug and expended significant human effort in the process. 

\textbf{Limitations.} To plan for a set of tasks, COIL must know the composition of that set beforehand. This prior knowledge will be difficult to come by reliably in all real-world settings. Additionally, the human's task preferences are assumed to be discrete and stationary, but prior work reveals that a human's preferences take a variety of form and are subject to change over time.
Furthermore, we use prior knowledge about similarities between tasks to predict the generalization of skills to future tasks. In many complex domains, such a prior may be inaccurate. For example, while the robot may be confident that its learned skill can be executed again on the next instance of the same skill, perturbations in the environment (e.g., slight differences in object orientation), unobserved environment variables (e.g., lighting), and other factors may cause the robot to fail when rolling out its learned skill.

\section{Conclusion} \label{sec:conclusion}
Teams are best positioned to succeed when each member accounts for their teammate's abilities and preferences. Robot teammates will need to do the same if human-robot teams are to succeed, and COIL is our proposed means of enabling this capability. By planning over multiple interaction types, contributions from both human and robot, and the human's associated task preferences and contribution costs, COIL identifies plans which minimize the burden placed upon the human on its way to achieving task success. COIL does so by formulating the interaction as a facility location problem which identifies the optimal sequence of robot actions and human contributions and then deducing if that optimal plan could be improved with more certain beliefs about the human teammate's preferences. COIL identifies plans which cost less as compared to baselines and efficiently scales to long task sequences. Moreover, COIL can learn human preferences for both task completion and teammate task contributions, and it can re-plan when skills prove too difficult to learn. Finally, COIL demonstrates its effectiveness in a real-world conveyor domain, inspired by its potential application in collaborative factories of the future. 
In our future work, we are interested in extending our planner to multi-step tasks. One possibility would be to decompose a multi-step tasks into a sequence of sub-tasks which can then be handled by our planner. Another direction of interest is to plan for unordered sets of tasks which introduces an additional challenge of scheduling.

\section*{Acknowledgments}
We thank Prof. Reid Simmons and Prof. Henny Admoni for their valuable feedback, Prof. Anupam Gupta for insightful discussions and Lakshita Dodeja for her help with our real-world robot experiments.
This work was supported by the Office of Naval Research (ONR) under REPRISM MURI N000142412603 and ONR grant \#N00014-22-1-2592, as well as by the National Science Foundation (NSF) via grant \#1955361. Partial funding was also provided by the Robotics and AI Institute.


\bibliographystyle{plainnat}
\bibliography{references,zotero}  

\ifshowappendix
    \begin{appendices}\label{sec:app}

\section{Bayesian Preference Belief Estimator}\label{sec:pref_filter}
We model the human's response $x_i$ to the robot's request for user preferences on task $\tau_i$ using a Bayesian belief estimator. Let $\seq{b}{i}{1}{N}$ be the set of beliefs for the sequence of tasks at timestep $t$, with $b^i$ the current beliefs for task $\tau_i$. The updated beliefs are
$$b^{'i}(\theta') := b^{i}(\theta'|x_i) \propto b^{i}(\theta')\mathbb{P}[x_i|\theta'] $$
which serve as updated beliefs for $\tau_i$ after receiving the human's response. The domain-specific likelihood $\mathbb{P}[x|\theta']$ is defined in Section \ref{ap_sec:skill_return}. The robot posits similar objects receive similar treatment and updates its beliefs for future similar tasks using the similarity function between tasks. That is, $\forall j\in (i,N)$, if $f(\lVert \tau_i - \tau_j\rVert) < \epsilon$, we update $b^{'j}(\theta') := b^{j}(\theta'|x) \propto b^{j}(\theta')\mathbb{P}[x_i|\theta']$.


\section{Environment Details}
\label{ap_sec:envs}
\begin{figure}[h]
\centering
\includegraphics[width=0.6\columnwidth]{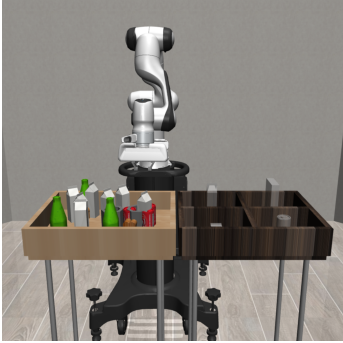}
\caption{Manipulation domain. Each interaction involves picking 50 objects and placing them in their respective bins.}
\end{figure}

In this section, we describe in detail the environments used in our experimental evaluation.
\begin{enumerate}
\item \textbf{Gridworld} is a discrete, 17x17 grid comprised of a sequence of 15 total objects of nine varieties distinguished by object \emph{type}, \emph{color}, and \emph{position}. Each object defines a task the agent must execute, and each task is to navigate to an object, pick it up, and transport it to the user's preferred location for that object. There are three possible goal locations. The 15 objects are randomly drawn from a set of 9 unique objects, whose distribution may vary. The robot is initialized with an empty skill library $\mathcal{L}_0$. We report results for 30 randomly initialized task sequences.
\item \textbf{7DoF Manipulation} is a simulated decluttering task where a Franka robot manipulator must pick up and put away 30 total objects of seven different varieties (e.g., milk carton, mug) into one of four different shelves. The frequency and order of all the objects are randomly sampled to generate a task sequence. Each object defines a task the agent must execute. The robot is initialized with an empty skill library $\mathcal{L}_0$.
\item \textbf{Conveyor} is a real-world environment where a Franka tabletop manipulator needs to pick up the object on the conveyor and place it into one of the three colored boxes to which it belongs. The robot is initialized with an empty skill library $\mathcal{L}_0$. We randomly sample five task sequences consisting of 20 tasks each. Since many objects are common across these experiments, we use a pretrained policy as the expert demonstrator to avoid redundant effort in the experiments.  We pretrain GEM~\cite{jia2024open} using human demonstrations for every task. GEM provides language-conditioned few-shot learning ability via $SE(2)$-equivariance and open-vocabulary relevancy maps. 
\end{enumerate}

\section{Human Cost Profiles}
\label{sec:humancostprofiles}
In our evaluations, we examined the behavior of \coil compared with baselines under a suite of human teaching cost profiles. We draw from prior literature on teaching with different feedback types \cite{fitzgeraldINQUIREINteractiveQuerying2022} to assign reasonable costs to each. We intuit robot execution is minimally burdensome and refer to prior works which find that demonstrations are more costly than preference queries \cite{koppol2021interaction}. The question is for different users, how much more costly is a demonstration than a preference request, and how do those costs compare to the cost of the human executing the task themselves. As these relative costs may differ across users, we investigate the performance of each algorithm under various demonstration (skill teaching) cost profiles. We assign a cost of $c_{rob}=10$ to each robot execution, $c_{hum}=80$ to human task execution, and $c_{pref}=20$ to each preference query. The robot incurs a penalty cost of $c_{fail}=100$ if it fails to successfully execute any skill, and a penalty of $c_{pref-fail}=100$ if it successfully executes a skill by placing the object in a goal undesired by the user. We examine profiles across the cost of human skill teaching.
\begin{enumerate}
    \item \emph{Low-Cost Teaching} ($c_{skill}=50$): This profile simulates an experienced teacher for whom providing demonstrations of robot skills is not burdensome. 
    \item \emph{Med-Cost Teaching} ($c_{skill}=100$): This profile models teaching to be moderately more burdensome than performing the task themselves.
    \item \emph{High-Cost Teaching} ($c_{skill}=200$): This profile simulates a novice teacher for whom providing demonstrations of robot skills is highly burdensome. 
\end{enumerate}

\section{Baseline Details}
\label{sec:baselines}
In this section, we describe in detail the baselines used in our experimental evaluation.

\subsection{Information Gain (IG).}
At timestep $t$, with the current task $\tau_i$, \info reasons over actions in set $A = \{pref, skill, rob, hum\}$.

\info takes action $a^{\tau_i}$ at task $\tau_i$ where:
\begin{equation}
    \begin{split}
        a^{\tau^i}&=\arg\max_{a \in A} \bigg[ IG^{\tau_i}(a)\ind_{pref} + SG^{\tau_i}(a)\ind_{skill} \bigg] - \beta c({a})
    \end{split}
\end{equation}
where $\ind_{pref}$ is a indicator variable equal to 1 if $a=pref$, and so on. Effectively, this objective reasons about (1) the information gain in preference space, $IG$, of requesting a preference, and (2) the information gain in skill space, or skill gain $SG$ of requesting a skill demonstration. For each potential robot action, the objective compares the potential gain relative to the cost of the query, where cost is scaled by factor $\beta$ to enable composition. We tune the $\beta$ hyperparameter, with details in Section \ref{subsec:ig_hyp}.

Let $\seq{b}{i}{1}{N}$ be the set of beliefs for the sequence of tasks at timestep $t$. Using a Bayesian belief update, let $\seq{b^g}{i}{1}{N}$ be the set of updated beliefs for the sequence of tasks at timestep $t$ when response $g$ is given to query $a$ on task $\tau_i$. In turn, the optimistic entropy reduction when the best potential choice $g$ is made given query $a$ can be written as $IG^{\tau_i}(a)=\max_g \sum_{j=i}^N H(b^g_{j})-H(b_i)$.

Skill gain is computed similarly and can be reasoned about as information gain over skills. Let $\mathcal{L}_t$ be the current skill library. Let $\mathcal{L}^g_{t} = \mathcal{L}_t \cup \tau_i$ be the updated skill library should a demonstration to goal $g$ be requested on $\tau_i$. Skill gain, then, can be represented as $SG^{\tau_i}(a)=\max_g \sum_{j=i}^N \mathbb{E}_{b_j^{\theta}} \bigg[\rho_{\pi, \mathcal{L}^g_t}^\text{safe}(\tau_j, \theta)-\rho_{\pi, \mathcal{L}_t}^\text{safe}(\tau_j, \theta) \bigg]$, where $\rho_{\pi, \mathcal{L}_t}^\text{safe}(\tau_j, \theta')$ is the likelihood of safe execution on $\tau_j$ given skill library $\mathcal{L}_t$ (computed by a lookup into $\mathcal{L}_t$).

Lastly, $\beta$ represents a scale factor needed to scale the costs in order to make them comparable with information gain and skill gain values. Let $\hat{\theta}_j := \arg\max_{\theta'}b^{\theta}_j$ be the maximum a posteriori estimate for the robot's belief on the preference for task $\tau_j$. The costs for each query follow from the costs used by COIL:
\begin{enumerate}
    \item $c({pref}) = c^i_\text{pref}$
    \item $c({rob}) = c^i_\text{rob} +c_\text{skill-fail} \cdot (1 - \rho_{\pi, \mathcal{L}_t}^\text{safe}(\tau_j, \hat{\theta}_j)) +c_\text{pref-fail} \cdot (1 - b^\theta_j(\hat{\theta}_j))$
    \item $c({skill}) = c^i_\text{skill} + c^i_\text{rob} + c_\text{skill-fail} \cdot (1-\rho_{\pi, \mathcal{L}_t}^\text{safe}(\tau_j, \hat{\theta}_j))$
    \item $c({hum}) = c^i_{hum}$
\end{enumerate}

\subsection{Confidence-based ADL (C-ADL)}
\citet{vats2022synergistic} use a mixed integer program to decide when to learn new skills and when to delegate tasks to the user. While their method accounts for the future utility of learning a task, they do so without also considering the human's preferences. In order to ensure a fair comparison in our paradigm, we enable \cadl to request preference queries for each task with a confidence-based threshold. We use a confidence threshold of $\alpha=0.8$. If the robot isn't sufficiently confident about the user preference, the robot requests a preference. Once received, the robot acts according to its ADL plan with $\hat{\theta}_j := \arg\max_{\theta'}b^{\theta}_j$ the maximum a posteriori estimate for the robot's belief on the preference for task $\tau_j$. 

\subsection{Confidence-based Autonomy (CBA).}
Inspired by~\citet{chernovaConfidencebasedPolicyLearning2007}, \cba requests skill and preference teaching if the robot is uncertain about either. Otherwise, it assigns the task to itself or the human. We use a confidence threshold of $\alpha=0.8$. If the robot isn't sufficiently confident about the user preference, the robot requests the user's preference. Let $\hat{\theta}_j := \arg\max_{\theta'}b^{\theta}_j$ be the maximum a posteriori estimate for the robot's belief on the preference for task $\tau_j$. Next, the robot requests a skill query if it isn't sufficiently confident on the execution of the skill. Given current skill library at timestep $t$, recall $\mathcal{L}_t$, $\rho_{\pi, \mathcal{L}_t}^\text{safe}(\tau_j, \hat{\theta}_j)$ is the likelihood of task success on $\tau_j$ given skill library $\mathcal{L}_t$, which can be computed by a lookup into $\mathcal{L}_t$. If both are confident ($\rho_{\pi, \mathcal{L}_t}^\text{safe}(\tau_j, \hat{\theta}_j)>\alpha$), then the robot executes its skill with the most likely preference. Important to note, this is a learning-prioritizing baseline, as it reasons through all learning options, and doesn't reason over delegation to the human.



\subsection{Information Gain Hyperparameter Sensitivity Analysis}
\label{subsec:ig_hyp}
The \info baseline relies on a scale factor $\beta$ which scales the human teaching costs to be evaluated jointly with the information gain. The choice of this $\beta$ scale factor influences the performance of the algorithm, as a factor too high will focus only on picking the action with the minimal cost, disregarding information gain. Prior to running our comparisons, we first running a sensitivity analysis on \info for different values of $\beta \in [0.001, 0.01, 0.05, 0.1, 0.5, 1.0]$. In Figure \ref{fig:ig_hyperparameters}, we plot the average plan cost for 10 seeds per each of the three human cost profiles. For \info in our comparisons, we choose the scale factor which achieves the lowest plan cost, $\beta=0.01$.
\begin{figure}[!tb]
    \centering
    \includegraphics[width=0.8\linewidth]{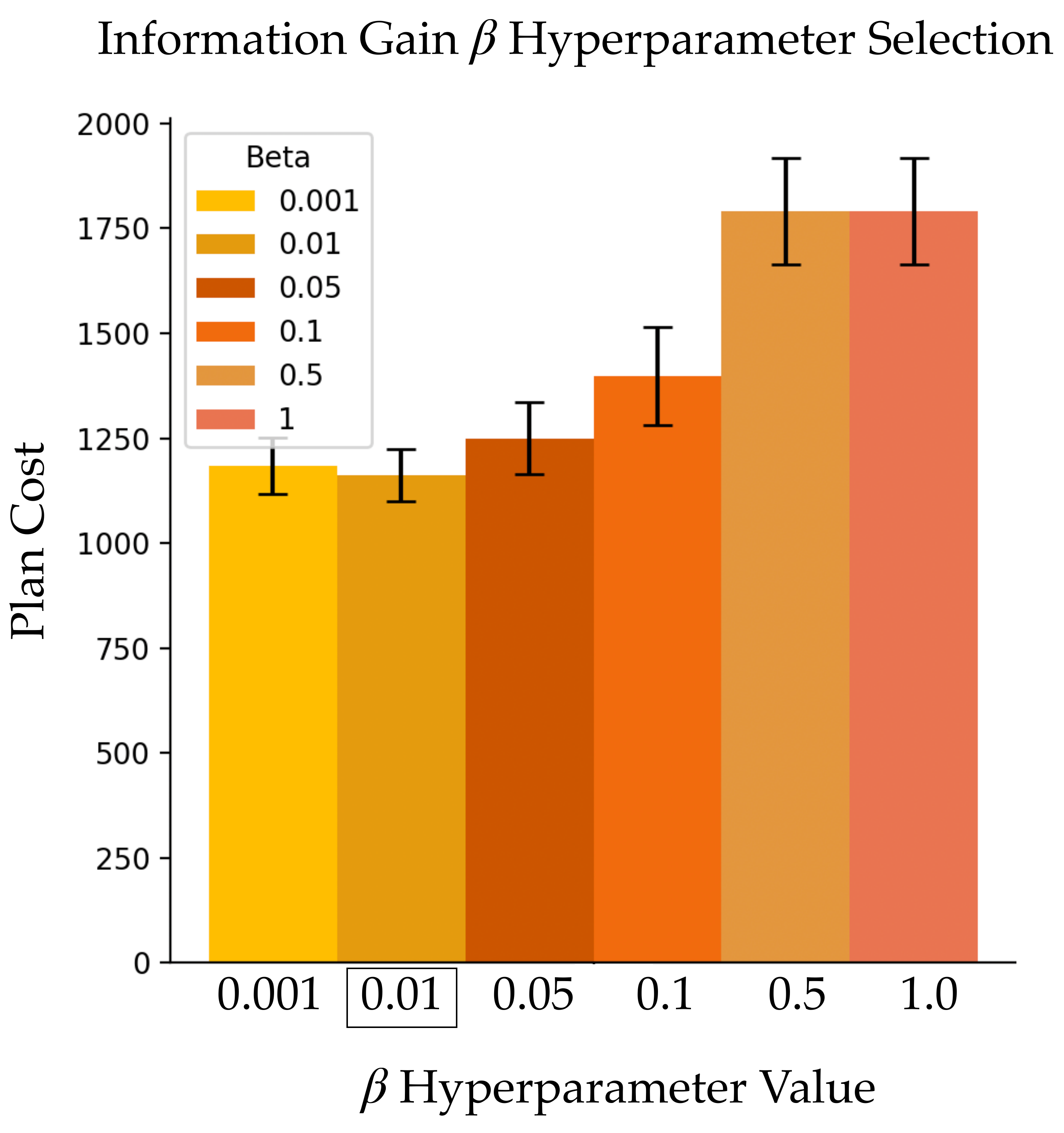}
    \caption{We plot the average plan cost for 10 seeds per all of the three human cost profiles. For \info in our comparisons, we choose the scale factor which achieves the lowest plan cost, $\beta=0.01$.}
    \label{fig:ig_hyperparameters}
\end{figure}

\section{Domain-specific Instantiations of Skill Return Model and Preference Beliefs}
\label{ap_sec:skill_return}

\subsection{Gridworld}
We provide the robot with a domain-specific function $\rho_{\pi}^\text{safe}$ to predict the probability of safe execution. $\rho_\pi^\text{safe}(\tau', \theta')$ predicts the probability of safe execution of a skill $\pi$ when deployed on a task $\tau'$ with parameter $\theta'$. 
This function is defined based on similarity between $(\tau', \theta')$ and the tasks currently in the robot's skill library $\mathcal{L}$. For the Gridworld domain, we define $\rho_\pi^\text{safe}(\tau', \theta')=1$ if $\lVert \tau - \tau'\rVert < \epsilon$ for any $\tau \in \mathcal{L}$, where $\lVert \tau - \tau'\rVert=0$ if the object at $\tau$ and $\tau'$ are of the same type and color, and 0 otherwise. $\epsilon$ is a small constant 0.01. 

The set of beliefs for the sequence of tasks is initialized $\seq{b}{i}{1}{N}$ is initialized as a uniform distribution over Bernoulli random variables indicating whether each goal position is suitable for the task object. Upon receiving a human response for where the object should be placed, the likelihood function for Bernoulli random variable associated with the goal is given a probability of 1, with the others 0, which is used in the Bayes update to update the robot's preference beliefs.

\subsection{Simulated 7DoF Manipulation}
The manipulation domain has $7$ different types of objects: milk carton, bread loaf, cereal box, can, bottle, lemon and cup. Each object is described by its type, color, dimensions and pose. 
We define  $\rho_\pi^\text{safe}(\tau', \theta')=1$ if $\lVert \tau - \tau'\rVert < \epsilon$ and $\theta$ is equal to $\theta'$, where $\tau$ and $\theta$ are the task and preference respectively that $\pi$ was trained for. $\lVert \tau - \tau'\rVert=0$ if the object at $\tau$ and $\tau'$ are of the same type (e.g. milk carton), and 0 otherwise. $\epsilon$ is a small constant 0.01. This is based on the assumption that the skills taught to the robot can generalize to different colors and object poses. We use the same similarity function in the preference belief estimator to update the beliefs of similar future objects.

\subsection{Conveyor}
The conveyor domain has $12$ different types of objects: pink bottle, green bottle, white mug, red mug, brown tape, black tape, orange block, blue block, banana, lemon, can and spam. Each object is described by its  type, category $\in \{\text{office}, \text{kitchen}, \text{toys}\}$, color, dimensions and pose. 
We define  $\rho_\pi^\text{safe}(\tau', \theta')=1$ if $\lVert \tau - \tau'\rVert < \epsilon$ and $\theta$ is equal to $\theta'$, where $\tau$ and $\theta$ are the task and preference respectively that $\pi$ was trained for. $\lVert \tau - \tau'\rVert=0$ if the object at $\tau$ and $\tau'$  of the same type (e.g. tape), and 0 otherwise. $\epsilon$ is a small constant 0.01. This is based on the assumption that the skills taught to the robot can generalize to variations in color, dimensions and pose. The preference belief estimator uses the object category to update the beliefs of other objects. This is based on the assumption that the human is likely to have similar preferences for objects of the same category.

\subsection{Gridworld Under All Teachable Objects: Statistical Analyses}
\label{ap_sec:statistical}
Our experiments under varied human cost profiles, we evaluated our null hypothesis of no significance in the sample (N=30) of total costs between the planning approaches. We compared statistical differences using a one-way ANOVA \cite{st1989analysis} between costs achieved for each algorithm, and evaluated pairwise differences using pairwise t-tests \cite{kim2015t} with Bonferroni correction \cite{weisstein2004bonferroni} if significant main effects were present. For the Low-Cost human profile, we did not find significant differences between the approaches.

For the Med-Cost human cost profile in the Gridworld domain, we found significant effects of planner ($F=23.48, p<0.001$), and found significant ($p<0.01$) pairwise differences between \coil and all baseline algorithms. \coil achieved significantly lower cost plans than \cadl ($F=-3.54, p<0.001$), \info ($F=-7.13, p<0.001$), and \cba ($F=-7.12, p<0.001$). \cadl also achieved significantly lower plan costs than \info ($F=-4.14, p<0.001$) and \cba ($F=-4.14, p<0.001$).

For the High-Cost human cost profile in the Gridworld domain, we found significant effects of planner ($F=120.58, p<0.001$), and found significant ($p<0.01$) pairwise differences between \coil versus \info ($F=-5.92, p<0.001$) and \cba ($F=-15.07, p<0.001$). \cadl also achieved significantly lower plan costs than \info ($F=-4.43, p<0.001$) and \cba ($F=-14.07, p<0.001$).

\section{Additional Results on Adapting to Teaching Failures in the Gridworld Domain}
\label{ap_sec:failures}
We similarly compared statistical differences using a one-way ANOVA between costs achieved for each algorithm, and evaluated pairwise differences using pairwise t-tests with Bonferroni correction if significant main effects were present. 

\subsection{Adapting when 10\% of Objects are Challenging-to-Learn}
We did not find significant effects of planner in the Low-Cost profile.
Under Med-Cost in the Gridworld domain, we found significant effects of planner ($F=14.54, p<0.001$), and found significant ($p<0.01$) pairwise differences between \coil versus \info and \cba. We did not find statistical significance between \coil and \coilnoadapt plan costs, but observe a slight decrease in plan cost in Figure \ref{fig:gridworld_unteachable10}. \coil achieved significantly lower cost plans than \info ($F=-5.73, p<0.001$), and \cba ($F=-5.73, p<0.001$). \coilnoadapt also achieved significantly lower plan costs than \info ($F=-1.25, p<0.001$) and \cba ($F=-4.90, p<0.001$). \cadl outperformed \info ($F=-3.70, p<0.001$) and \cba ($F=-3.70, p<0.001$). 
\begin{figure}[!tb]
    \centering
    \includegraphics[width=0.99\columnwidth]{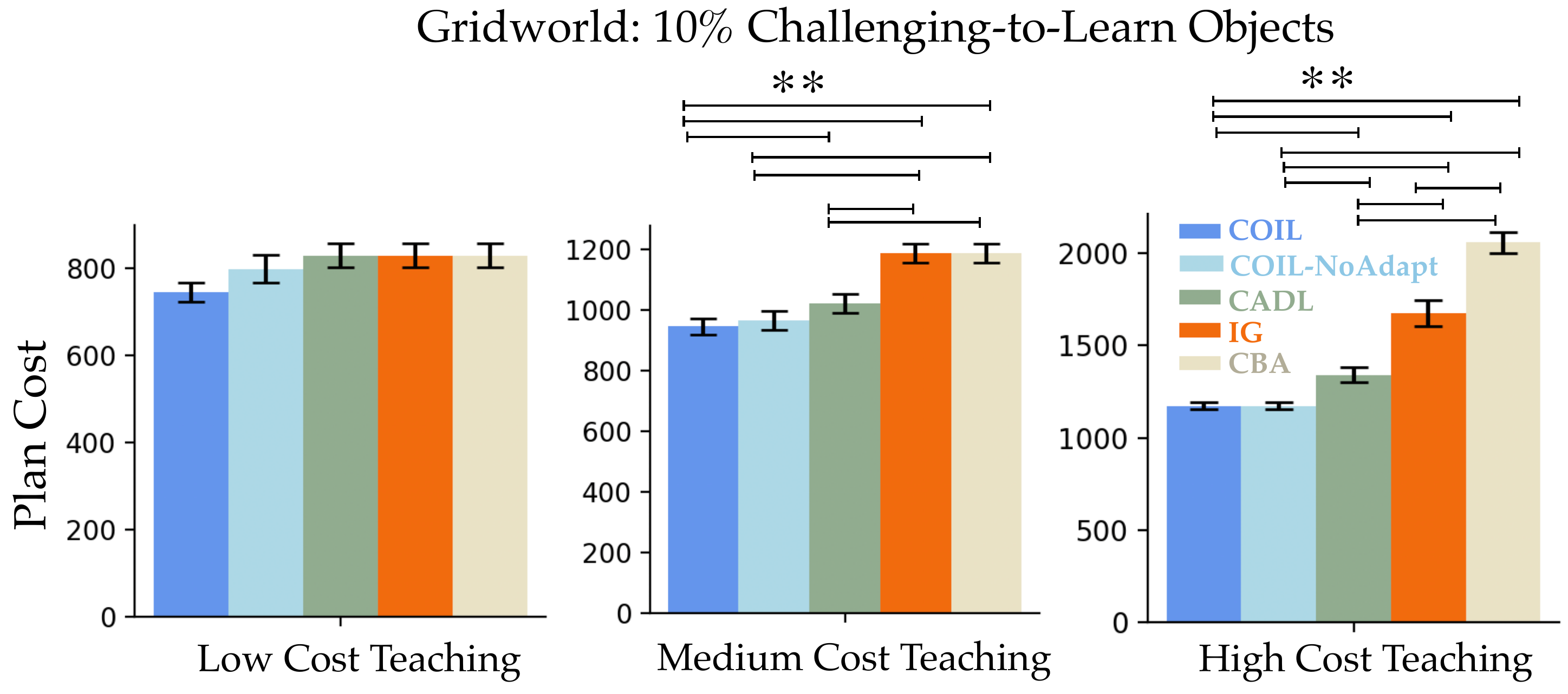}
    \caption{In the Gridworld environment, when 10\% of the objects are challenging-to-learn, \coil achieves the lowest plans, compared to \coilnoadapt and other baselines in all teaching profiles. In high-cost and med-cost teaching, \coil and \coilnoadapt often assign tasks to the human off the bat, reducing the impact of adaptivity. Error bars represent standard error over N=30.}
    \label{fig:gridworld_unteachable10}
\end{figure}

Under High-Cost in the Gridworld domain, we found significant effects of planner ($F=222.61, p<0.001$), and found significant pairwise differences between \coil versus \cadl, \info and \cba. We did not find statistical significance between \coil and \coilnoadapt plan costs. \coil achieved significantly lower cost plans than \cadl($F=-6.55, p<0.001$), \info ($F=-20.32, p<0.001$), and \cba ($F=-20.25, p<0.001$). \coilnoadapt also achieved significantly lower plan costs than \cadl ($F=-6.47, p<0.001$), \info ($F=20.24, p<0.001$) and \cba ($F=-20.24, p<0.001$). \cadl outperformed \info ($F=-13.08, p<0.001$) and \cadl ($F=--13.35, p<0.001$), and \info outperformed \cba ($F=-0.55, p<0.001$).

\subsection{Adapting when 50\% of Objects are Challenging-to-Learn}
Under Low-Cost in the Gridworld domain, we found significant effects of planner ($F=18.57, p<0.001$), and found significant pairwise differences between \coil compared to all baselines. We found statistical significance between \coil and \coilnoadapt plan costs, with \coil achieving a lower plan cost than \coilnoadapt ($F=-4.57, p<0.001$). \coil achieved significantly lower cost plans than \cadl($F=-7.72, p<0.001$), \info ($F=-7.72, p<0.001$), and \cba ($F=-7.72, p<0.001$). We did not find additional pairwise significances between the other planners.

Under Med-Cost in the Gridworld domain when half of the unique objects are challenging-to-learn, we found significant effects of planner ($F=50.03, p<0.001$), and found significant pairwise differences between \coil versus \cadl, \info and \cba. We did not find statistical significance between \coil and \coilnoadapt plan costs. \coil achieved significantly lower cost plans than \cadl($F=-4.91, p<0.001$), \info ($F=-11.74, p<0.001$), and \cba ($F=-11.74, p<0.001$). \coilnoadapt also achieved significantly lower plan costs than \info ($F=-8.73, p<0.001$) and \cba ($F=-8.73, p<0.001$), but had no difference with \cadl. \cadl outperformed \info ($F=-5.63, p<0.001$) and \cba ($F=-5.63, p<0.001$).

Under High-Cost, we found significant effects of planner ($F=127.82, p<0.001$), and found significant pairwise differences between \coil versus \cadl, \info and \cba. We did not find statistical significance between \coil and \coilnoadapt plan costs. \coil achieved significantly lower cost plans than \cadl($F=-6.20, p<0.001$), \info ($F=-12.15, p<0.001$), and \cba ($F=-18.80, p<0.001$). \coilnoadapt also achieved significantly lower plan costs than \cadl ($F=-6.20, p<0.001$), \info ($F=-12.15, p<0.001$) and \cba ($F=-18.80, p<0.001$). \cadl outperformed \info ($F=-7.11, p<0.001$) and \cba ($F=-11.74, p<0.001$), and \info outperformed \cba ($F=-3.02, p<0.001$).

\subsection{Adapting when 90\% of Objects are Challenging-to-Learn}
We lastly increased the proportion of challenging-to-learn objects to nearly all object varieties to evaluate how each planner influenced the interaction (Figure \ref{fig:gridworld_unteachable90}).
\begin{figure}[!tb]
    \centering
    \includegraphics[width=0.99\columnwidth]{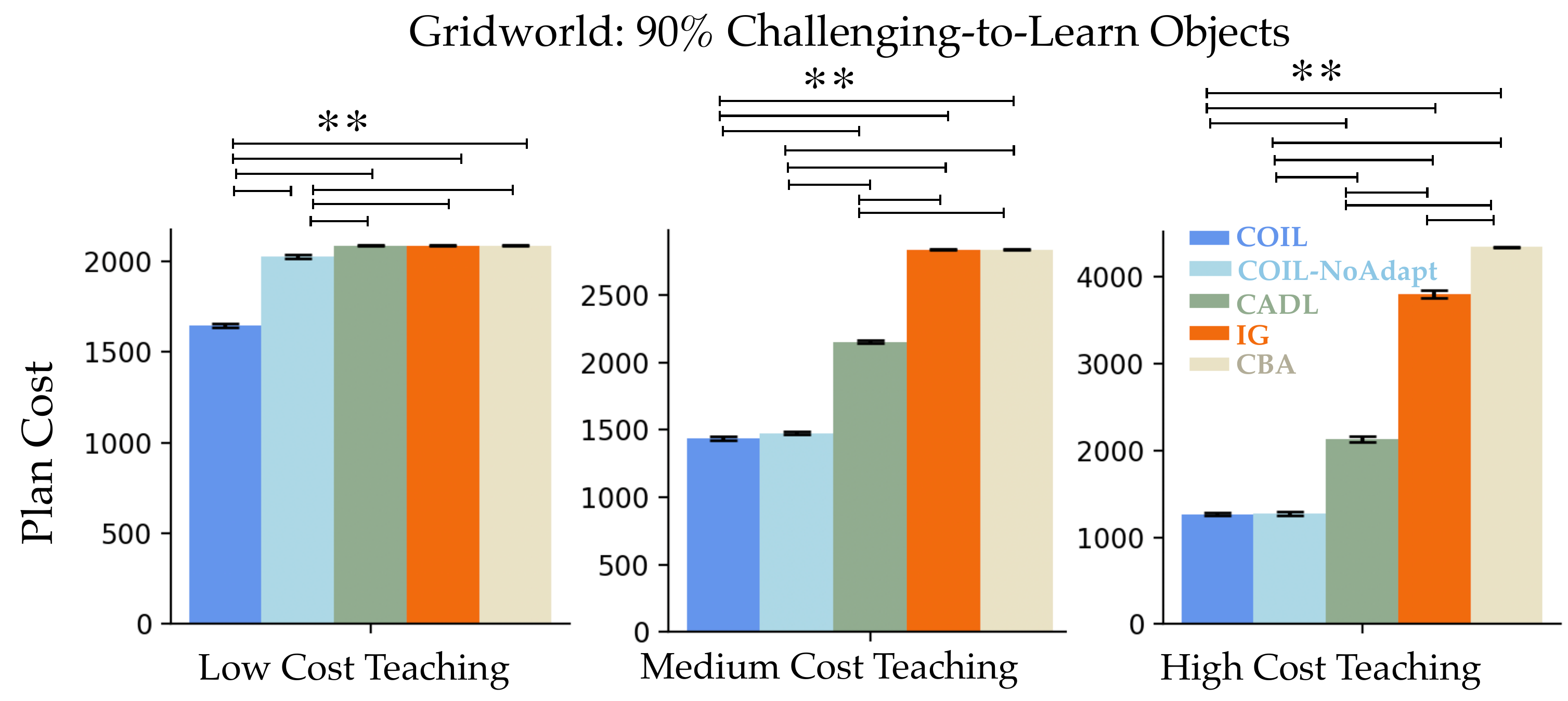}
    \caption{In the Gridworld environment, when 90\% of the objects are challenging-to-learn, \coil achieves the lowest plans, compared to \coilnoadapt and other baselines in all teaching profiles. In high-cost and med-cost teaching, \coil and \coilnoadapt often assign tasks to the human off the bat, reducing the impact of adaptivity. Error bars represent standard error over N=30.}
    \label{fig:gridworld_unteachable90}
\end{figure}
Under Low-Cost, we found significant effects of planner ($F=918.38, p<0.001$), and found significant pairwise differences between \coil versus all baselines. \coil achieved significantly lower cost plans than \coilnoadapt($F=28.35, p<0.001$), \cadl($F=-40.92, p<0.001$), \info ($F=-40.92, p<0.001$), and \cba ($F=-40.92, p<0.001$). \coilnoadapt also achieved significantly lower plan costs than \cadl ($F=-6.90, p<0.001$), \info ($F=-6.90, p<0.001$) and \cba ($F=-6.90, p<0.001$). We did not find additional pairwise significances between the other planners.

Under Med-Cost, we found significant effects of planner ($F=4513.26, p<0.001$), and found significant pairwise differences between \coil versus \cadl, \info and \cba. We did not find statistical significance between \coil and \coilnoadapt plan costs. \coil achieved significantly lower cost plans than \cadl($F=-41.10, p<0.001$), \info ($F=-100.08, p<0.001$), and \cba ($F=-100.08, p<0.001$). \coilnoadapt also achieved significantly lower plan costs than \cadl ($F=-37.36, p<0.001$), \info ($F=-92.01, p<0.001$) and \cba ($F=-92.01, p<0.001$). \cadl outperformed \info ($F=-62.45, p<0.001$) and \cba ($F=-62.45, p<0.001$).

Under High-Cost, we found significant effects of planner ($F=2596.18, p<0.001$), and found significant pairwise differences between \coil versus \cadl, \info and \cba. We did not find statistical significance between \coil and \coilnoadapt plan costs. \coil achieved significantly lower cost plans than \cadl($F=-23.25, p<0.001$), \info ($F=-51.29, p<0.001$), and \cba ($F=-169.35, p<0.001$). \coilnoadapt also achieved significantly lower plan costs than \cadl ($F=-22.16, p<0.001$), \info ($F=-49.97, p<0.001$) and \cba ($F=-145.20, p<0.001$). \cadl outperformed \info ($F=-29.50, p<0.001$) and \cba ($F=-67.20, p<0.001$), and \info outperformed \cba ($F=-11.71, p<0.001$).

\section{Simulated 7DoF Manipulation: Statistical Analyses}
\label{ap_sec:manip_statistical}
We similarly compared statistical differences using a one-way ANOVA between costs achieved for each algorithm, and evaluated pairwise differences using pairwise t-tests with Bonferroni correction if significant main effects were present. In this experiment, we compare with our 3 strongest baselines, derived from our earlier results, removing \cba from our analysis.

Under Low-Cost, we found significant effects of planner ($F=996.49, p<0.001$), and found significant pairwise differences between \coil versus all baselines. \coil achieved significantly lower cost plans than \coilnoadapt($F=-6.99, p<0.001$), \cadl($F=-11.10, p<0.001$), \info ($F=-60.68, p<0.001$). \coilnoadapt also achieved significantly lower plan costs than \cadl ($F=-3.41, p<0.001$), and \info ($F=-49.19, p<0.001$). \cadl outperformed \info ($F=-36.55, p<0.001$).

Under Med-Cost, we found significant effects of planner ($F=4445.39, p<0.001$), and found significant pairwise differences between \coil versus all baselines. \coil achieved significantly lower cost plans than \coilnoadapt($F=-11.99, p<0.001$), \cadl($F=-11.16, p<0.001$), \info ($F=-138.80, p<0.001$). \coilnoadapt also achieved significantly lower plan costs than \cadl ($F=-2.35, p<0.001$), and \info ($F=-111.16, p<0.001$). \cadl outperformed \info ($F=-87.31, p<0.001$).

Under High-Cost, we found significant effects of planner ($F=166.95, p<0.001$), and found significant pairwise differences between \coil versus all baselines. \coil achieved significantly lower cost plans than \coilnoadapt($F=-15.71, p<0.001$), \cadl($F=-19.03, p<0.001$), \info ($F=-36.22, p<0.001$). \coilnoadapt also achieved significantly lower plan costs than \cadl ($F=-4.60, p<0.001$), but had no difference when compared to \info. \cadl underperformed \info ($F=4.29, p<0.001$).

\end{appendices}
\fi
\end{document}